%% file: 00_main.tex
\documentclass{article}

\input{000_premeable.tex}

\makeatletter
\def\url@leostyle{%
  \@ifundefined{selectfont}{\def\UrlFont{\sf}}{\def\UrlFont{\small\sffamily}}}
  
\def\url@fleostyle{%
  \@ifundefined{selectfont}{\def\UrlFont{\sf}}{\def\UrlFont{\scriptsize\sffamily}}}
  
\patchcmd\@combinedblfloats{\box\@outputbox}{\unvbox\@outputbox}{}{%
   \errmessage{\noexpand\@combinedblfloats could not be patched}%
}%
\makeatother
\usepackage[accepted]{icml2020}

\icmltitlerunning{}
\begin{document}

\urlstyle{leo}

\twocolumn[
\icmltitle{Learning Adversarially Robust Representations via \\ Worst-Case Mutual Information Maximization}

\icmlsetsymbol{equal}{*}

\begin{icmlauthorlist}
\icmlauthor{Sicheng Zhu}{uva,equal}
\icmlauthor{Xiao Zhang}{uva,equal}
\icmlauthor{David Evans}{uva}
\end{icmlauthorlist}

\icmlaffiliation{uva}{Department of Computer Science, University of Virginia}

\icmlcorrespondingauthor{Sicheng Zhu}{sz6hw@virginia.edu}
\icmlcorrespondingauthor{Xiao Zhang}{xz7bc@virginia.edu}
\icmlcorrespondingauthor{David Evans}{evans@virginia.edu}

\icmlkeywords{Machine Learning, ICML}

\vskip 0.3in
]
\printAffiliationsAndNotice{\icmlEqualContribution} 
% \printAffiliationsAndNotice{} 

\input{01_abstract.tex}

\input{02_intro.tex}

\input{03_preliminaries.tex}

\input{04_representation_rob.tex}

\input{05_estimate.tex}

\input{06_train.tex}

\input{07_experiments.tex}

\input{08_conclusion.tex}

\section*{Acknowledgements}

This work was partially funded by an award from the National Science Foundation SaTC program (Center for Trustworthy Machine Learning, \#1804603) and additional support from Amazon, Baidu, and Intel.

\bibliography{robust_feature}
\bibliographystyle{icml2020}

\newpage
\onecolumn

\input{09_appendix.tex}

\end{document}

%% file: 000_premeable.tex
\usepackage{microtype}
\usepackage{graphicx}
\usepackage{subfigure}
\usepackage{booktabs} 
\usepackage{multicol}

\usepackage{hyperref}
\usepackage{smile}
\usepackage{comment}
\usepackage{url}

\usepackage{amsmath}
\usepackage{amssymb} 
\usepackage{tabularx}
\usepackage{caption}
\usepackage{xcolor}
\usepackage{mathtools}

\usepackage{siunitx}
\usepackage{etoolbox}
\robustify\bfseries

\usepackage[T1]{fontenc}

\renewcommand{\algorithmiccomment}[1]{#1}

\DeclareMathOperator{\E}{\mathbb{E}}

\def \cD {\mathcal{D}}
\def \cX {\mathcal{X}}
\def \cY {\mathcal{Y}}
\def \cZ {\mathcal{Z}}

\def \AdvRisk {\mathrm{AdvRisk}}
\def \Risk {\mathrm{Risk}}

\def \RV {\mathrm{RV}}
\def \AG {\mathrm{AG}}
\def \sgn {\mathrm{sgn}}
\def \I {\mathrm{I}}
\def \H {\mathrm{H}}
\def \W {\mathrm{W}}

\def \bin {\text{bin}}

\newcommand\shortsection[1]{\vspace{6pt}{\noindent\bf #1.}}

\DeclareMathOperator{\esssup}{{\rm -ess\:sup}}

%% file: 01_abstract.tex
\begin{abstract}

Training machine learning models that are robust against adversarial inputs poses seemingly insurmountable challenges. To better understand adversarial robustness, we consider the underlying problem of learning robust representations. We develop a notion of \textit{representation vulnerability} that captures the maximum change of mutual information between the input and output distributions, under the worst-case input perturbation. Then, we prove a theorem that establishes a lower bound on the minimum adversarial risk that can be achieved for any downstream classifier based on its representation vulnerability. We propose an unsupervised learning method for obtaining intrinsically robust representations by maximizing the worst-case mutual information between the input and output distributions. Experiments on downstream classification tasks %and analyses of saliency maps 
support the robustness of the representations found using unsupervised learning with our training principle. 

\end{abstract}

%% file: 02_intro.tex
\section{Introduction}

Machine learning has made remarkable breakthroughs in many fields, including computer vision \citep{he2016deep} and natural language processing \citep{devlin2018bert}, especially when evaluated on classification accuracy on a given dataset. However, adversarial vulnerability \citep{szegedy2013intriguing,engstrom2017rotation}, remains a serious problem that impedes the deployment of the state-of-the-art machine learning models in safety-critical applications, such as autonomous driving \citep{eykholt2018robust} and face recognition \citep{sharif2016accessorize}.
Despite extensive efforts to improve model robustness, state-of-the-art adversarially robust training methods \citep{madry2017towards,zhang2019theoretically} still fail to produce robust models, even for simple classification tasks on CIFAR-10 \citep{krizhevsky2009learning}.

In addition to many ineffective empirical attempts for achieving model robustness, recent studies have identified intrinsic difficulties for learning in the presence of adversarial examples. For instance, a line of works \citep{gilmer2018adversarial,fawzi2018adversarial,mahloujifar2019curse,shafahi2018adversarial} proved that adversarial vulnerability is inevitable if the underlying input distribution is concentrated. \citet{schmidt2018adversarially} showed that for certain learning problems, adversarially robust generalization requires more sample complexity compared with standard one, whereas \citet{bubeck2018adversarial} constructed a specific task on which adversarially robust learning is computationally intractable.

Motivated by the apparent empirical and theoretical difficulties of robust learning with adversarial examples, we focus on the underlying problem of learning adversarially robust representations \citep{garg2018spectral,pensia2019extracting}.
Given an input space $\cX\subseteq\RR^d$ and a feature space $\cZ\subseteq\RR^n$, any function $g:\cX\rightarrow\cZ$ is called a representation with respect to $(\cX,\cZ)$.  Adversarially robust representations denote the set of functions from $\cX$ to $\cZ$ that are less sensitive to adversarial perturbations with respect to some metric $\Delta$ defined on $\cX$. 
Note that one can always get an overall classification model by learning a downstream classifier given a representation, thus learning representations that are robust can be viewed as an intermediate step for the ultimate goal of finding adversarially robust models. In this sense, learning adversarially robust representations may help us better understand adversarial examples, and perhaps more importantly, bypass some of the aforementioned intrinsic difficulties for achieving model robustness.

In this paper, we give a general definition for robust representations based on mutual information, then study its implications on model robustness for a downstream classification task. Finally, we propose empirical methods for estimating and inducing representation robustness. 

%\newpage

\shortsection{Contributions}
Motivated by the empirical success 
of standard representation learning using the mutual information maximization principle \citep{bell1995information,hjelm2018learning}, we first give a formal definition on \emph{representation vulnerability} as the maximum change of mutual information between the representation's input and output against adversarial input perturbations bound in an $\infty$-Wasserstein ball (Section \ref{sec:robust feature}). Under a Gaussian mixture model, we established theoretical connections between the robustness of a given representation and the adversarial gap of the best classifier that can be based on it (Section \ref{sec:gaussian mix}). In addition, based on the standard mutual information and the  representation vulnerability, we proved a fundamental lower bound on the minimum adversarial risk that can be achieved for any downstream classifiers built upon a representation with given representation vulnerability (Section \ref{sec:general case}). 

To further study the implication of robust representations, we first propose a heuristic algorithm to empirically estimate the vulnerability of a given representation (Section \ref{sec:estimation}), and then by adding a regularization term on representation vulnerability in the objective of mutual information maximization principle, provide an unsupervised way for training meaningful and robust representations (Section \ref{sec:training}). We observe a direct correlation between model and representation robustness in experiments on benchmark image datasets MINST and CIFAR-10 (Section \ref{section:exp1}). Experiments on downstream classification tasks and saliency maps further show the effectiveness of our proposed training method in obtaining more robust representations (Section \ref{section:exp2}).

\shortsection{Related Work}
With similar motivations, several different definitions of robust features have been proposed in literature. The pioneering work of \citet{garg2018spectral} considered a feature to be robust if it is insensitive to input perturbations in terms of the output values. However, their definition of feature robustness is not invariant to scale changes. Based on the linear correlation between feature outputs and true labels, \citet{ilyas2019adversarial} proposed a definition of robust features to understand adversarial examples, whereas \citet{eykholt2019robust} proposed to study robust features whose outputs will not change with respect to small input perturbations. However, these two definitions either require the additional label information or restrict the feature space to be discrete, thus are not general. The most closely related work to ours is \citet{pensia2019extracting}, which considered Fisher information of the output distribution as the indicator of feature robustness and proposed a robust information bottleneck method for extracting robust features. Compared with \citet{pensia2019extracting}, our definition is defined for the worst-case input distribution perturbation, whereas Fisher information can only capture feature's sensitivity near the input distribution. In addition, our proposed training method for robust representations is better in the sense that it is unsupervised.

\shortsection{Notation}
We use small boldface letters such as $\bx$ to denote vectors and capital letters such as $X$ to denote random variables. 
Let $(\cX,\Delta)$ be a metric space, where $\Delta:\cX\times\cX\rightarrow\RR$ is some distance metric. Let $\cP(\cX)$ denote the set of all probability measures on $\cX$ and $\delta_{\bx}$ be the dirac measure at $\bx\in\cX$. Let $\cB(\bx,\epsilon, \Delta)= \{\bx'\in\cX: \Delta(\bx',\bx)\leq \epsilon\}$ be the ball around $\bx$ with radius $\epsilon$. When $\Delta$ is free of context, we simply write $\cB(\bx,\epsilon) = \cB(\bx,\epsilon, \Delta)$.
Denote by $\sgn(\cdot)$ the sign function such that $\sgn(x)=1$ if $x\geq 0$; $\sgn(x)=-1$ otherwise.
Given $f:\cX\rightarrow\cY$ and $g:\cY\rightarrow\cZ$, define $g\circ f$ as their composition such that for any $\bx\in\cX$, $(g\circ f) (\bx) = g(f(\bx))$. 
We use $[m]$ to denote $\{1,2,\ldots,m\}$ and $|\cA|$ to denote the cardinality of a finite set $\cA$.
For any $\bx\in\RR^d$, the $\ell_p$-norm of $\bx$ is defined as $\|\bx\|_p = (\sum_{i\in[d]}x_i^p)^{1/p}$ for any $p\geq 1$.
For any $\btheta\in\RR^d$ and positive definite matrix $\bSigma\in\RR^{d\times d}$, denote by $\cN(\btheta,\bSigma)$ the $d$-dimensional Gaussian distribution with mean vector $\btheta$ and covariance matrix $\bSigma$.

%% file: 03_preliminaries.tex
\section{Preliminaries}

This section introduces the main ideas we build upon: mutual information, Wasserstein distance and adversarial risk.

\shortsection{Mutual information} Mutual information is an entropy-based measure of the mutual dependence between variables:

\begin{definition}
Let $(X,Z)$ be a pair of random variables with values over the space $\cX\times\cZ$. The \emph{mutual information} of $(X,Z)$ is defined as:
\begin{equation*}
    \I(X;Z) = \int_\cZ\int_\cX p_{XZ}(\bx,\bz) \log\bigg(\frac{p_{XZ}(\bx,\bz)}{p_X(\bx) p_Z(\bz)} \bigg) d\bx d\bz,
\end{equation*}
where $p_{XZ}$ is the joint probability density function of $(X,Z)$, and $p_{X},p_Z$ are the marginal probability density functions of $X$ and $Z$, respectively.
\end{definition}

Intuitively, $\I(X; Z)$ tells us how well one can predict $Z$ from $X$ (and $X$ from $Z$, since it is symmetrical). By definition, $\I(X;Z)=0$ if $X$ and $Z$ are independent; when $X$ and $Z$ are identical, $\I(X;X)$ equals to the entropy $\H(X)$.

\shortsection{Wasserstein distance} Wasserstein distance is a distance function defined between two probability distributions on a given metric space:

\begin{definition}
Let $(\cX,\Delta)$ be a metric space with bounded support. Given two probability measures $\mu$ and $\nu$ on $(\cX,\Delta)$, the \emph{$p$-th Wasserstein distance}, for any $p\geq 1$, is defined as:
\begin{equation*}
    \mathrm{W}_p(\mu,\nu) = \bigg(\inf_{\gamma\in\Gamma(\mu,\nu)} \int_{\cX\times\cX} \Delta(\bx,\bx')^p\: d\gamma(x,x')\bigg)^{1/p},
\end{equation*}
where $\Gamma(\mu,\nu)$ is the collection of all probability measures on $\cX\times\cX$ with $\mu$ and $\nu$ being the marginals of the first and second factor, respectively. The \emph{$p$-th Wasserstein ball} with respect to $\mu$ and radius $\epsilon\geq 0$ is defined as:
\begin{equation*}
    \cB_{\W_p} (\mu,\epsilon) = \{\mu'\in\cP(\cX): \W_p(\mu',\mu)\leq\epsilon\}.
\end{equation*}
\end{definition}

The $\infty$-Wasserstein distance is defined as the limit of $p$-th Wasserstein distance, $\W_\infty(\mu,\nu)= \lim_{p\rightarrow\infty}\:\W_p(\mu,\nu)$.

\shortsection{Adversarial risk}
Adversarial risk captures the vulnerability of a given classification model to input perturbations:

\begin{definition}
\label{def:adv risk}
Let $(\cX,\Delta)$ be the input metric space and $\cY$ be the set of labels. Let $\mu_{XY}$ be the underlying distribution of the input and label pairs. For any classifier $f:\cX\rightarrow\cY$, the \emph{adversarial risk} of $f$ with respect to $\epsilon\geq 0$ is defined as:
\begin{align*}
    \AdvRisk_\epsilon(f) = \Pr_{(\bx,y)\sim\mu_{XY}} \big[\exists\: \bx'\in\cB(\bx,\epsilon) \text{ s.t. } 
    f(\bx')\neq y \big].
\end{align*}
\end{definition}

Adversarial risk with $\epsilon=0$ is equivalent to standard risk, namely $\AdvRisk_0(f) = \Risk(f) = \Pr_{(\bx,y)\sim\mu}[f(\bx)\neq y]$. For any classifier $f:\cX\rightarrow\cY$, we define the \emph{adversarial gap} of $f$ with respect to $\epsilon$ as:
\begin{align*}
    \mathrm{AG}_\epsilon(f) = \AdvRisk_\epsilon(f) -\Risk(f).
\end{align*}

%% file: 04_representation_rob.tex
\section{Adversarially Robust Representations}
\label{sec:robust feature}

In this section, we first propose a definition of representation vulnerability, and then prove several theorems that bound achievable model robustness based on representation vulnerability. Let $\cX\subseteq\RR^{d}$ be the input space and $\cZ\subseteq\RR^{n}$ be some feature space. In this work, we define a representation to be a function $g$ that maps any input $\bx$ in $\cX$ to some vector $g(\bx)\in\cZ$. A classifier, $f = h \circ g$, maps an input to a label in a label space $\cY$, and is a composition of a downstream classifier, $h: \cZ \rightarrow \cY$, with a representation, $g: \cX \rightarrow \cZ$.  As is done in previous works \citep{garg2018spectral,ilyas2019adversarial}, we define a feature as a function from $\cX$ to $\RR$, so can think of a representation as an array of features.

Inspired by the empirical success of standard representation learning using the mutual information maximization principle \citep{hjelm2018learning}, we propose the following definition of \emph{representation vulnerability}, which captures the robustness of a given representation against input distribution perturbations in terms of mutual information between its input and output. 

\begin{definition}
\label{def:feature rob}
Let $(\cX, \mu_X, \Delta)$ be a metric probability space of inputs and $\cZ$ be some feature space. Given a representation $g:\cX\rightarrow\cZ$ and $\epsilon\geq 0$, the \emph{representation vulnerability} of $g$ with respect to perturbations bounded in an $\infty$-Wasserstein ball with radius $\epsilon$ is defined as:
\begin{align*}
    \RV_{\epsilon}(g) = \sup_{\mu_{X'}\in\cB_{\W_\infty}(\mu_{X},\epsilon)}\: \big[ \I (X; g(X)) - \I (X'; g(X')) \big],
\end{align*}
where $X$ and $X'$ denote random variables that follow $\mu_X$ and $\mu_{X'}$, respectively. 
\end{definition}

Representation vulnerability is always non-negative, and higher values indicate that the representation is less robust to adversarial input distribution perturbations. More formally, given parameters $\epsilon\geq0$ and $\tau\geq0$, a representation $g$ is called \emph{$(\epsilon,\tau)$-robust} if $\RV_\epsilon(g)\leq \tau$.

Notably, using the $\infty$-Wasserstein distance does not restrict the choice of the metric function $\Delta$ of the input space. This metric $\Delta$ corresponds to the perturbation metric for defining adversarial examples. Thus, based on our definition of representation vulnerability, our following theoretical results and empirical methods work with any adversarial perturbation, including any $\ell_p$-norm based attack. 

Compared with existing definitions of robust features \citep{garg2018spectral, ilyas2019adversarial,  eykholt2019robust}, our definition is more general and enjoys several desirable properties. As it does not impose any constraint on the feature space, it is invariant to scale change\footnote{Scale-invariance is desirable for representation robustness. Otherwise, one can always divide the function by some large constant to improve its robustness,  e.g., \citet{garg2018spectral}.} and it does not require the knowledge of the labels.
The most similar definition to ours is from \citet{pensia2019extracting}, who propose to use statistical Fisher information as the evaluation criteria for feature robustness. However, Fisher information can only capture the average sensitivity of the log conditional density to small changes on the input distribution (when $\epsilon\rightarrow 0$), whereas our definition is defined with respect to the worst-case input distribution perturbations in an $\infty$-Wasserstein ball, which is more aligned with the adversarial setting. As will be shown next, our representation robustness notion has a clear connection with the potential model robustness of any classifier that can be built upon a representation.

\subsection{Gaussian Mixture}
\label{sec:gaussian mix}

We first study the implications of representation vulnerability under a simple Gaussian mixture model. We consider $\cX\subseteq\RR^d$ as the input space and $\cY = \{-1, 1\}$ as the space of binary labels. Assume $\mu_{XY}$ is the underlying joint probability distribution defined over $\cX\times\cY$, where all the examples $(\bx,y)\sim\mu_{XY}$ are generated according to
\begin{align}
\label{eq:GMM}
    y\sim \mathrm{Unif}\{-1,+1\}, \quad \bx\sim\cN(y\cdot\btheta^*, \bSigma^*),
\end{align}
where $\btheta^*\in\RR^d$ and $\bSigma^*\in\RR^{d\times d}$ are given parameters. The following theorem, proven in Appendix~\ref{proof:thm two Gaussian}, connects the vulnerability of a given representation with the adversarial gap of the best classifier based on the representation.

\begin{theorem}
\label{thm:two Gaussian}
Let $(\cX,\|\cdot\|_p)$ be the input metric space and $\cY = \{-1,1\}$ be the label space. Assume the underlying data are generated according to \eqref{eq:GMM}. Consider the feature space $\cZ=\{-1,1\}$ and the set of representations, 
$$
\cG_{\text{bin}} = \{g:\bx\mapsto\sgn(\bw^\top\bx),\forall \bx\in\cX \:\big|\; \|\bw\|_2=1\}.
$$ 
Let $\cH=\{h:\cZ\rightarrow\cY\}$ be the set of non-trivial downstream classifiers.\footnote{To be more specific, we do not consider the case where $h$ is a constant function. Under our problem setting, there are two elements in $\cH$, namely $h_1(z) = z$, $h_2(z) = -z$, for any $\bz\in\cZ$.}
Given $\epsilon\geq 0$, for any $g\in\cG_\bin$, we have 
\begin{align*}
    \int_{\frac{1}{2}-\AG_\epsilon(f^*)}^{\frac{1}{2}} \H_2'(\theta)d\theta \leq \RV_\epsilon(g) \leq \int_{\frac{1}{2}-\frac{1}{2}\AG_\epsilon(f^*)}^{\frac{1}{2}} \H_2'(\theta)d\theta,
\end{align*}
where $f^*=\argmin_{h\in\cH}\AdvRisk_{\epsilon}(h\circ g)$ is the optimal classifier based on $g$, $\H_2(\theta) = -\theta\log \theta - (1-\theta)\log(1-\theta)$ is the binary entropy function and $\H_2'$ denotes its derivative.
\end{theorem}

For this theoretical model for a simple case, Theorem \ref{thm:two Gaussian} reveals the strong connection between representation vulnerability and the adversarial gap achieved by the optimal downstream classifier based on the representation. Note that the binary entropy function $\H_2(\theta)$ is monotonically increasing over $(0,1/2)$, thus the first inequality suggests that low representation vulnerability guarantees a small adversarial gap if we train the downstream classifier properly. On the other hand, the second inequality implies that adversarial robustness cannot be achieved for any downstream classifier, if the vulnerability of the representation it uses is too high. As discussed in Section~\ref{section:exp1}, the connection between representation vulnerability and adversarial gap is also found to hold empirically for image classification benchmarks.

\subsection{General Case}
\label{sec:general case}

This section presents our main theoretical results regarding robust representations. First, we present the following lemma, proven in Appendix \ref{proof:lemma equal}, that characterizes the connection between adversarial risk and input distribution perturbations bounded in an $\infty$-Wasserstein ball.

\begin{lemma}
\label{lem:equal}
Let $(\cX,\Delta)$ be the input metric space and $\cY$ be the set of labels. Assume all the examples are generated from a joint probability distribution $(X,Y)\sim\mu_{XY}$. Let $\mu_X$ be the marginal distribution of $X$. Then, for any classifier $f:\cX\rightarrow\cY$ and $\epsilon>0$, we have
\begin{align*}
    \AdvRisk_\epsilon(f) = \sup_{\mu_{X'}\in\cB_{\W_\infty}(\mu_X,\epsilon)}\Pr\big[f(X')\neq Y\big],
\end{align*}
where $X'$ denotes the random variable that follows $\mu_{X'}$.
\end{lemma} 

The next theorem, proven in Appendix~\ref{proof:general thm}, gives a lower bound for the adversarial risk for any downstream classifier, using the worst-case mutual information between the representation's input and output distributions. 

\begin{theorem}
\label{thm:general lower bound}
Let $(\cX,\Delta)$ be the input metric space, $\cY$ be the set of labels and $\mu_{XY}$ be the underlying joint probability distribution. Assume the marginal distribution of labels $\mu_Y$ is a uniform distribution over $\cY$. Consider the feature space $\cZ$ and the set of downstream classifiers $\cH=\{h:\cZ\rightarrow\cY\}$. Given $\epsilon\geq 0$, for any $g:\cX\rightarrow\cZ$, we have
\begin{align*}
\inf_{h\in\cH} \AdvRisk_\epsilon(h\circ g) \geq 1 - \frac{\I(X;Z)-\RV_\epsilon(g) + \log2}{\log|\mathcal{Y}|},
\end{align*}
where $X$ is the random variable that follows the marginal distribution of inputs $\mu_X$ and $Z=g(X)$.
\end{theorem}
 
Theorem \ref{thm:general lower bound} suggests that adversarial robustness cannot be achieved if the available representation is highly vulnerable or the standard mutual information between $X$ and $g(X)$ is low. Note that $\I(X;g(X))-\RV_\epsilon(g) = \inf\{\I(X';g(X')): X'\sim \mu_{X'}\in\cB_{\W_\infty}(\mu_X,\epsilon)\}$, which corresponds to the worst-case mutual information between input and output of $g$. Therefore, if we assume robust classification as the downstream task for representation learning, then the representation having high worst-case mutual information is a necessary condition for achieving adversarial robustness for the overall classifier. 

In addition, we remark that Theorem \ref{thm:general lower bound} can be extended to general $p$-th Wasserstein distances, if the downstream classifiers are evaluated based on robustness under distributional shift\footnote{See \citet{sinha2018certifying} for a rigorous definition of distributional robustness.}, instead of adversarial risk. To be more specific, if using $\W_p$ metric to define representation vulnerability, we can then establish an upper bound on the maximum distributional robustness with respect to the considered $\W_p$ metric for any downstream classifier based on similar proof techniques of Theorem \ref{thm:general lower bound}.

%% file: 05_estimate.tex
\section{Measuring Representation Vulnerability}
\label{sec:estimation}

This section presents an empirical method for estimating the vulnerability of a given representation using i.i.d. samples.
Recall from Definition \ref{def:feature rob}, for any $g:\cX\rightarrow\cZ$, the representation vulnerability of $g$ with respect to the input metric probability space $(\cX,\mu_X,\Delta)$ and $\epsilon\geq 0$ is defined as:
\begin{equation}
\label{eq:measureFS}
    \RV_{\epsilon}(g) = \underbrace{\I(X; g(X))}_{J_1} - \underbrace{\inf_{\mu_{X'}\in\cB_{\W_\infty}(\mu_X,\epsilon)} \I(X'; g(X'))}_{J_2}.
\end{equation}
To measure representation vulnerability, we need to compute both terms $J_1$ and $J_2$. However, the main challenge is that we do not have the knowledge of the underlying probability distribution $\mu_X$ for real-world problem tasks. Instead, we only have access to a finite set of data points sampled from the distribution. 
Therefore, it is natural to consider sample-based estimator for $J_1$ and $J_2$ for practical use.

The first term $J_1$ is essentially the mutual information between $X$ and $Z = g(X)$. A variety of methods have been proposed for estimating mutual information \citep{moon1995estimation,darbellay1999estimation,suzuki2008approximating, kandasamy2015nonparametric, moon2017ensemble}. The most effective estimator is the mutual information neural estimator (MINE) \citep{belghazi2018mine}, based on the dual representation of KL-divergence \citep{donsker1983asymptotic}:
\begin{align*}
\hat{\I}_m(X;Z) = \sup_{\theta\in\Theta}\: \E_{\hat\mu_{XZ}^{(m)}}[T_\theta] - \log\big(\E_{\hat\mu_X^{(m)}	\otimes \hat\mu_Z^{(m)}} [\exp(T_\theta)]\big),
\end{align*}
where $T_\theta:\cX\times\cZ\rightarrow\RR$ is the function parameterized
by a deep neural network with parameters $\theta\in\Theta$, and $\hat\mu_{XZ}^{(m)}$, $\hat\mu_X^{(m)}$ and $\hat\mu_Z^{(m)}$ denote the empirical distributions\footnote{Given a set of $m$ samples $\{\bx_i\}_{i\in[m]}$ from a distribution $\mu$, we let $\hat\mu^{(m)}=\frac{1}{m}\sum_{i\in[m]}\delta_{\bx_i}$ be the empirical measure of $\mu$.} of random variables $(X,Z)$, $X$ and $Z$ respectively, based on $m$ samples. In addition, \citet{belghazi2018mine} empirically demonstrates the superiority of the proposed estimator in terms of estimation accuracy and efficiency, and prove that it is strongly consistent: for all $\varepsilon>0$, there exists $M\in\ZZ$ such that for any $m\geq M$, $|\hat{\I}_m(X;Z)-\I(X;Z)|\leq \varepsilon$ almost surely.
Given the established effectiveness of this method, we implement MINE to estimate $\I(X;g(X))$ as the first step.

Compared with $J_1$, the second term $J_2$ is much more difficult to estimate, as it involves finding the worst-case perturbations on $\mu_X$ in a $\infty$-Wasserstein ball in terms of mutual information. As with the estimation of $J_1$, we only have a finite set of instances sampled from $\mu_X$. On the other hand, due to the non-linearity and the lack of duality theory with respect to the $\infty$-Wasserstein distance \citep{champion2008wasserstein}, it is inherently difficult to directly solve an $\infty$-Wasserstein constrained optimization problem, even if we work with the empirical distribution of $\mu_X$. To deal with the first challenge, we replace $\mu_X$ with its empirical measure $\hat\mu^{(m)}_X$ based on i.i.d. samples. Then, to avoid the need to search through the whole $\infty$-Wasserstein ball,
we restrict the search space of $\mu_{X'}$ to be the following set of empirical distributions:
\begin{align}
\label{eq:perturbation set}
    \cA(\cS,\epsilon) = \bigg\{\frac{1}{m}\sum_{i=1}^m \delta_{\bx'_i}\colon \bx'_i\in\cB(\bx_i,\epsilon) \: \forall i\in[m] \bigg\},
\end{align}
where $\cS=\{\bx_i: i\in[m]\}$ denotes the given set of $m$ data points sampled from $\mu_X$. Note that the considered set $\cA(\cS,\epsilon)\subseteq\cB_{\W_\infty}(\hat\mu^{(m)}_{X},\epsilon)$, since each perturbed point $\bx_i'$ is at most $\epsilon$-away from $\bx_i$. Finally, making use of the dual formulation of KL-divergence that is used in MINE, we propose the following empirical optimization problem for estimating $J_2$:
\begin{align}
\label{eq:worst case MI}
    \min_{\mu_{X'}}\: \hat{\I}_m\big(X';g(X')\big) \:\text{ s.t. }\: \mu_{X'}\in\cA(\cS,\epsilon),
\end{align}
where we simply set the empirical distribution $\hat\mu^{(m)}_{X'}$ to be the same as $\mu_{X'}$. In addition, we propose a heuristic alternating minimization algorithm to solve \eqref{eq:worst case MI} (see Appendix \ref{sec:alg code} for the pseudocode and a complexity analysis of the proposed algorithm). More specifically, our algorithm alternatively performs gradient ascent on $\theta$ for the inner maximization problem of estimating $\hat{\I}_m(X';g(X'))$ given $\mu_X'$, and searches for the set of worst-case perturbations on $\{\bx'_i: i\in[m]\}$ given $\theta$ based on projected gradient descent.

%% file: 06_train.tex
\section{Learning Robust Representations}
\label{sec:training}

In this section, we present our method for learning adversarially robust representations. First, we introduce the mutual information maximization principle for representation learning \citep{linsker1989generate,bell1995information}.
Mathematically, given an input probability distribution $\mu_X$ and a set of representations $\cG=\{g:\cX\rightarrow\cZ\}$, the maximization principle proposes to solve this optimization problem:
\begin{align}
\label{eq:mutual infomax obj}
    \max_{g\in\cG} \:\:\I\big(X; g(X)\big).
\end{align}

Although this principle has been shown to be successful for learning good representations under the standard setting \citep{hjelm2018learning}, it becomes ineffective when considering adversarial perturbations (see Table \ref{tab:main:cifar} for an illustration). Motivated by the theoretical connections between feature sensitivity and adversarial risk for downstream robust classification shown in Section \ref{sec:robust feature}, we stimulate robust representations by adding a regularization term based on representation vulnerability: % to \eqref{eq:mutual infomax obj}:
\begin{align}
\label{eq:robust MI obj}
\max_{g \in \mathcal{G}} \quad  \I(X;g(X)) - \beta \cdot \RV_{\epsilon}(g), 
\end{align}
where $\beta\geq 0$ is the trade-off parameter between $\I(X;g(X))$ and $\RV_{\epsilon}(g)$. When $\beta =0$, \eqref{eq:robust MI obj} is same as the objective for learning standard representations \eqref{eq:mutual infomax obj}.
Increasing the value of $\beta$ will produce representations with lower vulnerability, but may undesirably affect the standard mutual information $\I(g(X); X)$ if $\beta$ is too large. In particular, we set $\beta=1$ in the following discussions, which allows us to simplify \eqref{eq:robust MI obj} to obtain the following optimization problem:
\begin{align}
\label{eqn:objective}
\max_{g \in \mathcal{G}}\: \min_{\mu_{X'}\in\cB_{\W_\infty}(\mu_{X},\epsilon)}\: \I\big(X';g(X')\big).
\end{align}
The proposed training principle \eqref{eqn:objective} aims to maximize the mutual information between the representation's input and output under the worst-case input distribution perturbation bounded in a $\infty$-Wasserstein ball. We remark that optimization problem \eqref{eqn:objective} aligns well with the results of Theorem \ref{thm:general lower bound}, which shows the importance of the learned feature representation achieving high worst-case mutual information for a downstream robust classification task.

As with estimating the feature sensitivity in Section~\ref{sec:estimation}, we do not have access to the underlying $\mu_X$. However, the inner minimization problem is exactly the same as estimating the worst-case mutual information $J_2$ in \eqref{eq:measureFS}, thus we can simply adapt the proposed empirical estimator \eqref{eq:worst case MI} to solve \eqref{eqn:objective}. To be more specific, we reparameterize $g$ using a neural network with parameter $\psi\in\Psi$ and use the following min-max optimization problem:
\begin{align}
\label{eq:empirical train obj}
     \max_{\psi\in\Psi} \: \min_{\mu_{X'}\in\cA(\cS,\epsilon)} \hat{\I}_m\big(X';g_{\psi}(X')\big). 
\end{align}
Based on the proposed algorithm for the inner minimization problem, \eqref{eq:empirical train obj} can be efficiently solved using a standard optimizer, such as stochastic gradient descent.

%% file: 07_experiments.tex
\section{Experiments}\label{sec:experiments}

This section reports on experiments to study the implications of robust representations on benchmark image datasets. Instead of focusing directly improving model robustness, our experiments focus on understanding the proposed definition of robust representations as well as its implications.  Based on the proposed estimator in Section \ref{sec:estimation}, Section~\ref{section:exp1} summarizes experiments to empirically test the relationship between representation vulnerability and model robustness, by extracting internal representations from the state-of-the-art pre-trained standard and robust classification models. In addition, we empirically evaluate the general lower bound on adversarial risk presented in Theorem \ref{thm:general lower bound}. In Section \ref{section:exp2}, we evaluate the proposed training principle for learning robust representations on image datasets, and test its performance with comparisons to the state-of-the-art standard representation learning method in a downstream robust classification framework. We also visualize saliency maps as an intuitive criteria for evaluating representation robustness.

We conduct experiments on MNIST \citep{lecun-mnisthandwrittendigit-2010}, Fashion-MNIST \citep{fashion-mnist}, SVHN \citep{netzer2011reading}, and CIFAR-10 \citep{krizhevsky2009learning}, considering typical $\ell_\infty$-norm bounded adversarial perturbations for each dataset ($\epsilon=0.3$ for MNIST, $0.1$ for Fashion-MNIST, $4/255$ for SVHN, and $8/255$ for CIFAR-10).
We use the PGD attack \citep{madry2017towards} for both generating adversarial distributions in the estimation of worst-case mutual information and evaluating model robustness. To implement our proposed estimator \eqref{eq:worst case MI}, we adopt the \textit{encode-and-dot-product} model architecture in \citet{hjelm2018learning} and adjust it to adapt to different forms of representations. We leverage implementations from \citet{robustness2019package} and \citet{hjelm2018learning} in our implementation\footnote{\urlstyle{fleo} \url{https://github.com/schzhu/learning-adversarially-robust-representations}}. Implementation details are provided in Appendix~\ref{section:appendix implement}.

\subsection{Representation Robustness}
\label{section:exp1}

To evaluate our proposed definition on representation vulnerability and its implications for downstream classification models, we conduct experiments on image benchmarks using various classifiers, including VGG \citep{Simonyan15}, ResNet \citep{he2016deep}, DenseNet \citep{huang2017densely} and the simple convolutional neural network in \citet{hjelm2018learning} denoted as Baseline-H.

\begin{figure}[tb] % figure 1
\centering
\includegraphics[width=0.4\textwidth]{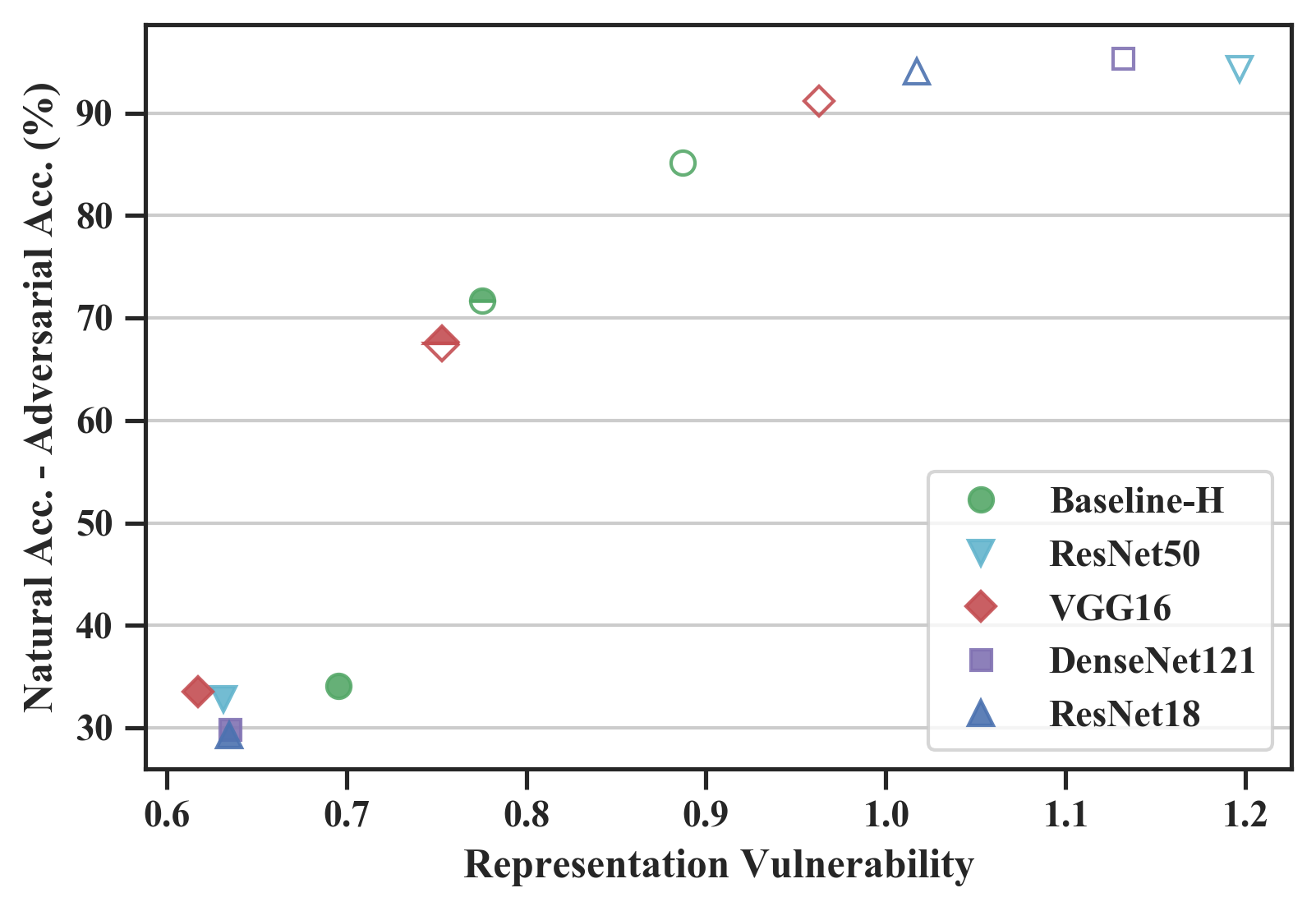}
%\vspace{-1em}
\caption{Correlations between the representation vulnerability and the CIFAR-10 model's natural-adversarial accuracy gap. Filled points indicate robust models (trained with $\epsilon=8/255$),
half-filled are models adversarially trained with $\epsilon=2/255$, and unfilled points are standard models.}
\label{fig:feature_model_robustness_delta}
\end{figure}

\shortsection{Correlation with model robustness}
Theorem~\ref{thm:two Gaussian} establishes a direct correlation between our representation vulnerability definition and achievable model robustness for the synthetic Gaussian-mixture case, but we are not able to theoretically establish that correlation for arbitrary distributions. Here, we empirically test this correlation on image benchmark datasets.
Figure \ref{fig:feature_model_robustness_delta} summarizes the results of these experiments for CIFAR-10, where we set the logit layer as the considered representation space.
The adversarial gap decreases with decreasing representation vulnerability in an approximately consistent relationship. 
Models with low logit layer representation vulnerability tend to have low natural-adversarial accuracy gap, which is consistent with the intuition behind our definition and with the theoretical result on the synthetic Gaussian-mixture case. This suggests the correlation between representation vulnerability and model robustness may hold for general case.

\begin{figure}[tb]
\centering
\includegraphics[width=0.4\textwidth]{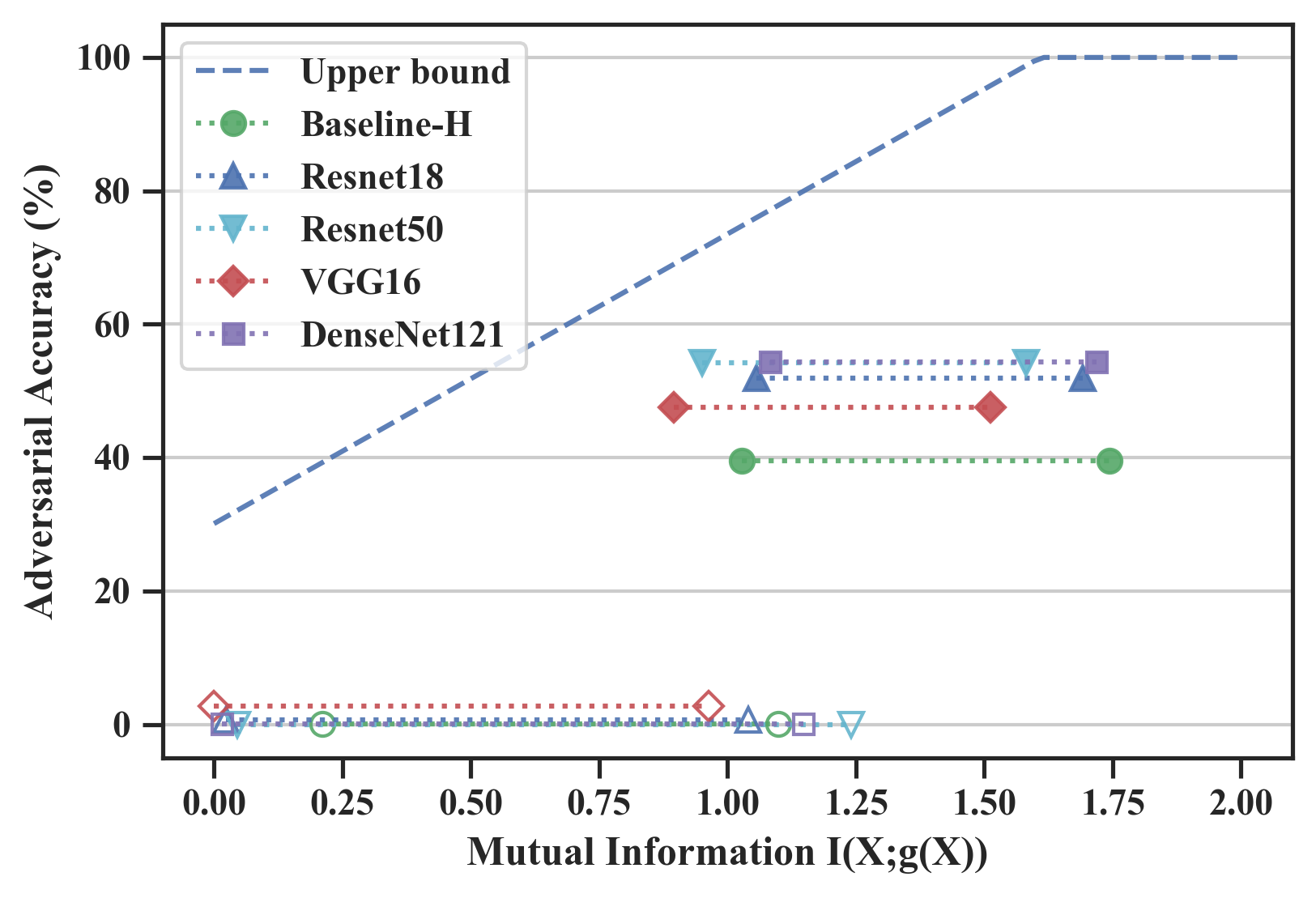}
%\vspace{-1em}
\caption{Normal and worst case mutual information for logit-layer representations. Each pair of points shows the result of a specific model---the left point indicates the worst case mutual information and the right for the normal mutual information. Filled points are robust models; hollow points are standard models.}
\label{fig:feature_model_robustness}
\end{figure}

\shortsection{Adversarial risk lower bound}
Theorem \ref{thm:general lower bound} provides a lower bound on the adversarial risk that can be achieved by any downstream classifier as a function of representation vulnerability. 
To evaluate the tightness of this bound, we estimate the normal-case and worst-case mutual information $\I(X; g(X))$ of layer representation $g$ for different models, and empirically evaluate the adversarial risk of the models.
Figure~\ref{fig:feature_model_robustness} shows the results, where we again set the logit layer as the feature space for a more direct comparison.
The lower bound of adversarial risk is calculated according to Theorem \ref{thm:general lower bound} and is converted to the upper bound of adversarial accuracy for reference. 
In particular, for standard models, both the estimated worst-case mutual information and the adversarial accuracy are close to zero, whereas the 
computed upper bounds on adversarial accuracy are around $30\%$.
We empirically observed around $50\%$ adversarial accuracy for robust models, whereas the bounds computed using the estimated worst-case mutual information and Theorem \ref{thm:general lower bound} are about $75\%$. 
This shows that Theorem~\ref{thm:general lower bound} gives a reasonably tight bound for a model's adversarial accuracy with respect to the logit-layer representation robustness. 

Figure~\ref{fig:feature_model_robustness} also indicates that even the robust models produced by adversarial training have representations that are not sufficiently robust to enable robust downstream classifications. For example, robust DenseNet121 in our evaluations has the highest logit layer worst-case mutual information of $1.08$, yet the corresponding adversarial accuracy is upper bounded by $77.0\%$ which is unsatisfactory for CIFAR-10. Such information theoretic limitation also justifies our training principle of worst-case mutual information maximization, since on the other hand the adversarial accuracy upper bound calculated by normal-case mutual information does not constitute a limitation for most robust models in our experiments (as in Figure~\ref{fig:feature_model_robustness_delta}, most robust models achieve adversarial accuracy close to $100\%$).

\shortsection{Internal feature robustness}
We further investigate the 
implications of our proposed definition from the level of individual features. Specifically for neural networks, we consider the function from the input to each individual neuron within a layer as a feature. The motivations for considering feature robustness comes from the fact that mutual information in terms of the whole representation is controlled by the sum of all the features' mutual information (see Appendix \ref{section:appendix:tensorization} for a rigorous argument) and robust features are potentially easier to train \citep{garg2018spectral}.
As an illustration, we evaluate the robustness of all the convolutional kernels in the second layer of the Baseline-H model. 
Each neuron evaluated here is a composite convolutional kernel (all kernels in the first layer connected to a second layer kernel) with image input size $10\times 10$.
Figure \ref{fig:neuron_historgram} shows the results that are averaged over two independently trained models for each type. 
This result reveals the apparent difference in feature robustness between a standard model and the adversarially-trained robust model, even in lower layers. 
Although in this case the result does not prohibit a robust downstream model for lower layers neurons, for neurons in higher layers the difference becomes more distinct and the vulnerability of neurons can thus be the bottleneck of achieving high model robustness.
The different feature robustness according to our definition also coincide with the saliency maps of features (see Figure~\ref{fig:saliencyKernel} in Appendix~\ref{section:appendix:additional results}), where the saliency maps of robust features are apparently more interpretable compared to those of standard features.

\begin{figure}[t]
\centering
\includegraphics[width=0.95\columnwidth]{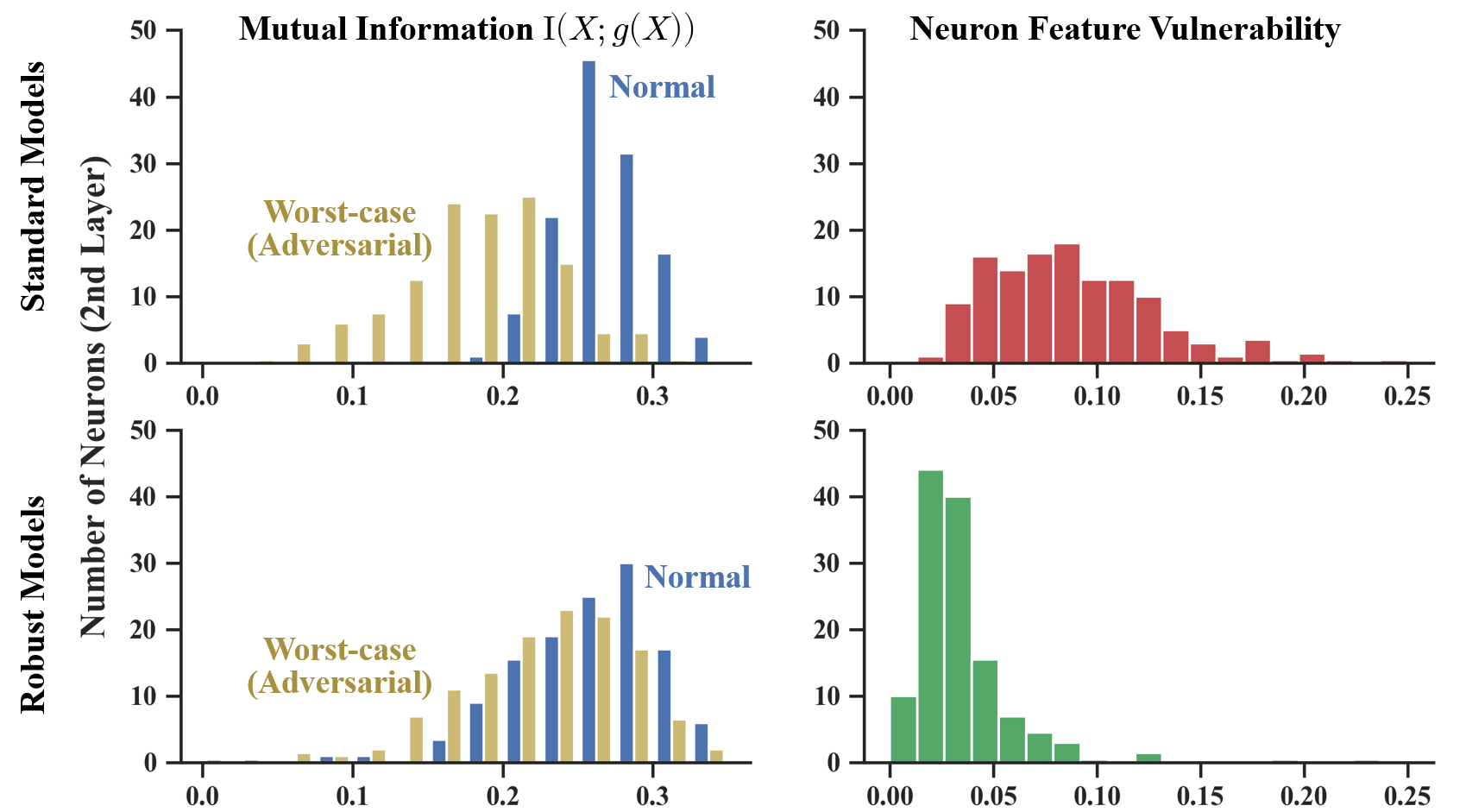}
%\vspace{-0.8em}
\caption{Distribution of mutual information $I(X;g(X))$ and feature vulnerability in the second convolutional layer of Baseline-H. The upper plots are for standard models, and the lower plots are for robust models. The total number of neurons is $128$.}
\label{fig:neuron_historgram}
\end{figure}

\begin{table*}[t]
\centering
\begin{tabular}{@{}llSSSS@{}}
\toprule
 &  & \multicolumn{2}{c}{MLP $h$} & \multicolumn{2}{c}{Linear $h$} \\ % \midrule
\multicolumn{1}{c}{\text{Representation ($g$)}} & \multicolumn{1}{c}{\text{Classifier ($h$)}} & \multicolumn{1}{c}{\text{Natural}} & \multicolumn{1}{c}{\text{Adversarial}} & \multicolumn{1}{c}{\text{Natural}} & \multicolumn{1}{c}{\text{Adversarial}} \\ \midrule
\multicolumn{1}{c}{\citet{hjelm2018learning}} & \multicolumn{1}{c}{Standard} & 58.77 \pm 0.22 & 0.22 \pm 0.08 & 47.01 \pm 0.53 & 0.15 \pm 0.03 \\
\multicolumn{1}{c}{\citet{hjelm2018learning}} & \multicolumn{1}{c}{Robust} & 29.75 \pm 1.49 & 15.08 \pm 0.63 & 22.79 \pm 1.42 & 10.28 \pm 0.52 \\
\multicolumn{1}{c}{Ours} & \multicolumn{1}{c}{Standard} & \bfseries 62.54 \pm 0.12 & 14.06 \pm 0.69 &\bfseries 50.29 \pm 0.58 & 10.98 \pm 0.49 \\
\multicolumn{1}{c}{Ours} & \multicolumn{1}{c}{Standard (E.S.)} & 51.59 \pm 3.34 & 27.53 \pm 0.81 & 48.55 \pm 0.63 & 13.52 \pm 0.16 \\
\multicolumn{1}{c}{Ours} & \multicolumn{1}{c}{Robust} & 52.34 \pm 0.17 & \bfseries 31.52 \pm 0.31 & 43.55 \pm 0.10 & \bfseries 25.15 \pm 0.10 \\ \midrule
\multicolumn{2}{c}{Fully-Supervised Standard} & 86.33 \pm 0.17 & 0.07 \pm 0.02 & 86.36 \pm 0.13 & 0.02 \pm 0.01 \\
\multicolumn{2}{c}{Fully-Supervised Robust} & 70.71 \pm 0.58 & 40.50 \pm 0.27 & 72.44 \pm 0.59 & 39.98 \pm 0.16 \\ \bottomrule
\end{tabular}
\caption{Comparisons of different representation learning methods on CIFAR-10 in downstream classification settings. \textit{E.S.} denotes early stopping under the criterion of the best adversarial accuracy. We present mean accuracy and the standard deviation over $4$ repeated trials. }
\label{tab:main:cifar}
\end{table*}

\subsection{Learning Robust Representations}
\label{section:exp2}

Our worst-case mutual information maximization training principle provides an unsupervised way to learn adversarially robust representations. 
Since there are no established ways to measure the robustness of a representation, empirically testing the robustness of representations learned by our training principle poses a dilemma. To avoid circular reasoning, we evaluate the learned representations by running a series of downstream adversarial classification tasks and comparing the performance of the best models we are able to find for each representation. In addition, recent work shows that the interpretability of saliency map has certain connections with robustness \citep{etmann2019connection,ilyas2019adversarial}, thus we study the saliency map as an alternative criteria for evaluating robust representations.

The unsupervised representation learning approach based on mutual information maximization principle in \citet{hjelm2018learning} achieves the state-of-the-art results in many downstream tasks, including standard classification.
We further adopt their encoder architecture in our implementation, and extend their evaluation settings to adversarially robust classification. Specifically, we truncate the front part of Baseline-H with a $64$-dimensional latent layer output as the representation $g$ and train it by the worst-case mutual information maximization principle using only unlabeled data (removing the labels from the normal training data). We test two architectures (two-layer multilayer perceptron and linear classifier) for implementing the downstream classifier $h$ and train it using labeled data after the encoder $g$ has been trained using unlabeled data. Appendix \ref{section:appendix implement} provides additional details on the experimental setup. 

\shortsection{Downstream classification tasks} 
Comparison results on CIFAR-10 are demonstrated in Table \ref{tab:main:cifar} (see Appendix \ref{section:appendix:additional results} for a similar results for MNIST, Fashion-MNIST, and SVHN). 
The fully-supervised models are trained for reference, from which we can see the simple model architecture we use achieves a decent natural accuracy of $86.3\%$; the adversarially-trained robust model reduces accuracy to around $70\%$ with adversarial accuracy of $40.5\%$.  The baseline, with $g$ and $h$ both trained normally, resembles the setting in \citet{hjelm2018learning} and achieves a natural accuracy of $58.8\%$.
For representations learned using worst-case mutual information maximization, the composition with standard two-layer multilayer perceptron (MLP) $h$ achieves a non-trivial (compared to the $0.2\%$ for the standard representation) adversarial accuracy of $14.1\%$. 
When $h$ is further trained using adversarial training, the robust accuracy increases to $31.5\%$ which is comparable to the result of the robust fully-supervised model. As an ablation, the robust $h$ based on standard $g$ achieves an adversarial accuracy of $15.1\%$, yet the natural accuracy severely drops below $30\%$, indicating that a robust classifier cannot be found using the vulnerable representation. The case where $h$ is a simple linear classifier shows similar results. These comparisons show that the representation learned using worst-case mutual information maximization can make the downstream classification more robust over the baseline and approaches the robustness of fully-supervised adversarial training. This provides evidence that our training principle produces adversarially robust representations.

Another interesting implication given by results in Table~\ref{tab:main:cifar} is that robustly learned representations may also have better natural accuracy ($62.5\%$) over the standard representation ($58.8\%$) in downstream classification tasks on CIFAR-10. This matches our experiments in Figure~\ref{fig:feature_model_robustness} where logit layer representations in robust models conveys more normal-case mutual information (up to $1.75$) than those in standard models (up to $1.25$). However, this is not the case on MNIST dataset as in Table~\ref{tab:main:mnist}. We conjecture that this is because the information conveyed by robust representations has better generalizations, and the generalization is more of a problem on CIFAR-10 than on MNIST \citep{schmidt2018adversarially}.

\shortsection{Saliency maps} 
A saliency map is commonly defined as the gradient of a model's loss with respect to the model's input \citep{etmann2019connection}. For a classification model, it intuitively illustrates what the model looks for in changing its classification decision for a given sample.
Recent work \citep{etmann2019connection,ilyas2019adversarial} indicates, at least in some synthetic settings, that the more alignment the saliency map has with the input image, the more adversarially robust the model is.
As an additional test of representation robustness, we calculate the saliency maps of standard and robust representations $g$ by the mutual information maximization loss with respect to the input.
Figure~\ref{fig:saliency:main} shows that the saliency maps of the robust representation appear to be much less noisy and more interpretable in terms of the alignment with original images.
Intuitively, this shows that robust representations capture relatively higher level visual concepts instead of pixel-level statistical clues \citep{engstrom2019learning}. The more interpretable saliency maps of representation learned by our training principle further support its effectiveness in learning adversarially robust representation.
\begin{figure}[tb]
\centering
\includegraphics[width=0.37\textwidth]{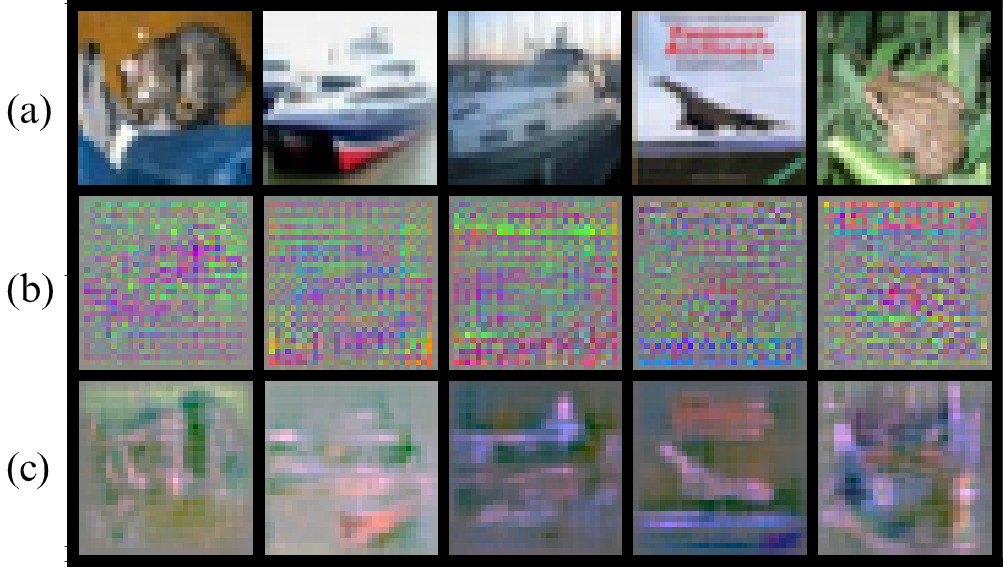}
\vspace{-0.5em}
\caption{Visualization of saliency maps of different models on CIFAR-10: (a) original images (b) representations learned using \citet{hjelm2018learning} (c) representations learned using our method.}
\label{fig:saliency:main}
\end{figure}

%% file: 08_conclusion.tex
\section{Conclusion}
We proposed a novel definition of representation robustness based on the worst-case mutual information, and showed both theoretical and empirical connections between our definition and model robustness for a downstream classification task. In addition, by developing estimation and training methods for representation robustness, we demonstrated the connection and the usefulness of the proposed method on benchmark datasets. Our results are not enough to produce strongly robust models, but they provide a new approach for understanding and measuring achievable adversarial robustness at the level of representations.

%% file: 09_appendix.tex
\appendix

\section{Proofs}

In this section, we provide the proofs of our theoretical results in Section \ref{sec:robust feature}. 

For ease of presentation, we introduce the following notations and an alternative definition of the $\infty$-Wasserstein distance.
Let $(\cX,\mu)$ and $(\cY,\nu)$ be two probability spaces. We say that $T : \cX \rightarrow \cY$ transports $\mu\in\cP(\cX)$ to $\nu\in\cP(\cY)$, and we call $T$ a
transport map, if
$\nu(B) = \mu(T^{-1}(B))$, for all $\nu$-measurable sets $B$.
In addition, for any measurable map $T:\cX\rightarrow\cY$, we define the \emph{pushforward} of $\mu$ through $T$ as $(T_{\#}(\mu))$ given by
$$
    (T_{\#}(\mu))(B) = \mu(T^{-1}(B)), \quad \text{for any measurable}\:\: B\subseteq\cY.
$$

\shortsection{Alternative definition of $\infty$-Wasserstein distance} From the perspective of transportation theory, given two probability measures $\mu$ and $\nu$ on $(\cX,\Delta)$, any joint probability distribution $\gamma\in\Gamma(\mu,\nu)$ corresponds to a specific transport map $T:\cX\rightarrow\cX$ that moves $\mu$ to $\nu$. Then, the $p$-th Wasserstein distance can viewed as finding the optimal transport map to move from $\mu$ to $\nu$ that minimizes some cost functional depending on $p$ \citep{kolouri2017optimal}. For the case where $p=\infty$, if we let $T$ be the transport map induced by a given $\gamma\in\Gamma(\mu,\nu)$, then the cost functional can be informally understood as the maximum of all the transport distances $\Delta(T(\bx),\bx)$. More rigorously, the $\infty$-Wasserstein distance can be alternatively defined as
\begin{align*}
    \W_\infty(\mu,\nu) &\coloneqq \inf_{\gamma \in \Gamma(\mu, \nu)}\: \underset{(\bx,\bx')\in\cX^2}{\gamma\esssup}\: \Delta(\bx,\bx') \\
    &= \inf_{\gamma \in \Gamma(\mu, \nu)}\: \inf\big\{t\geq 0\colon \gamma\big(\Delta(\bx,\bx')>t\big)=0\big\}.
\end{align*}
A more detailed discussion of $\infty$-Wasserstein distance can be found in \citet{champion2008wasserstein}.

\subsection{Proof of Theorem \ref{thm:two Gaussian}}
\label{proof:thm two Gaussian}

Theorem~\ref{thm:two Gaussian} (restated here) connects the vulnerability of a given representation with the minimum adversarial gap of any classifier based on that representation.

{\bf Theorem 3.2.\ }
Let $(\cX,\|\cdot\|_p)$ be the input metric space and $\cY = \{-1,1\}$ be the label space. Assume the underlying data are generated according to \eqref{eq:GMM}. Consider the feature space $\cZ=\{-1,1\}$ and the set of representations, 
$$
\cG_{\text{bin}} = \{g:\bx\mapsto\sgn(\bw^\top\bx),\forall \bx\in\cX \:\big|\; \|\bw\|_2=1\}.
$$ 
Let $\cH=\{h:\cZ\rightarrow\cY\}$ be the set of non-trivial downstream classifiers.\footnote{To be more specific, we do not consider the case where $h$ is a constant function. Under our problem setting, there are two elements in $\cH$, namely $h_1(z) = z$, $h_2(z) = -z$, for any $\bz\in\cZ$.}
Given $\epsilon\geq 0$, for any $g\in\cG_\bin$, we have 
\begin{align*}
    \int_{\frac{1}{2}-\AG_\epsilon(f^*)}^{\frac{1}{2}} \H_2'(\theta)d\theta \leq \RV_\epsilon(g) \leq \int_{\frac{1}{2}-\frac{1}{2}\AG_\epsilon(f^*)}^{\frac{1}{2}} \H_2'(\theta)d\theta,
\end{align*}
where $f^*=\argmin_{h\in\cH}\AdvRisk_{\epsilon}(h\circ g)$ is the optimal classifier based on $g$, $\H_2(\theta) = -\theta\log \theta - (1-\theta)\log(1-\theta)$ is the binary entropy function and $\H_2'$ denotes its derivative.

\begin{proof}
Let $\mu_{XY}$ be the underlying joint probability distribution of the examples according to \eqref{eq:GMM} and $\mu_X$ be corresponding the marginal distribution of $X$.
To begin with, we compute the explicit formulation for the defined representation vulnerability in Definition \ref{def:feature rob}. Note that $\I(U;V) = \H(U)-\H(U|V) = \H(V)-\H(V|U)$. Thus, for any $g_{\bw}\in\cG$ we have
\begin{align*}
    \RV_\epsilon(g_{\bw}) 
    &= \H(g_{\bw}(X)) - \inf_{\mu_{X'}\in \cB_{\W_\infty}(\mu_X,\epsilon)}\: \H(g_{\bw}(X')) \\
    &= 1 - \inf_{\mu_{X'}\in\cB_{\W_\infty}(\mu_{X},\epsilon)}\: \bigg( - \Pr_{\bx'\sim\mu_{X'}}(\bw^\top\bx'\geq 0) \cdot\log \Big[\Pr_{\bx'\sim\mu_{X'}}(\bw^\top\bx'\geq 0)\Big] \\
    &\qquad\qquad-\Pr_{\bx'\sim\mu_{X'}}(\bw^\top\bx'< 0 )\cdot\log \Big[\Pr_{\bx'\sim\mu_{X'}}(\bw^\top\bx'< 0)\Big] \bigg),
\end{align*}
where the first equality holds because $\H(g_{\bw}(U)\:|\:U)=0$ for any random variable $U$, and the second equality is due to the fact that the distribution of $X$ is symmetric with respect to $\bw^\top\bx = 0$. Note that the binary entropy function $\H_2(\theta) = -\theta\log \theta - (1-\theta)\log(1-\theta)$ is monotonically increasing with respect to $\theta$ in $[0,1/2)$ and monotonically decreasing in  $(1/2,1]$. Therefore, the optimal value of $\RV_\epsilon(g_{\bw})$ is achieved when $\mu_{X'}$ either minimizes $\Pr_{\bx'\sim\mu_{X'}}(\bw^\top\bx'\geq 0)$ or maximizes $\Pr_{\bx'\sim\mu_{X'}}(\bw^\top\bx'\geq 0)$.

According to the H\"older's inequality, we have $|\langle \ab, \bb \rangle|\leq \|\ab\|_p\cdot\|\bb\|_q$ for any $\ab,\bb\in\RR^d$, where $1/p+1/q=1$. By the alternative definition of $\infty$-Wasserstein distance, for any $\mu_{X'}$ that satisfies $\W_\infty(\mu_{X'},\mu_X)\leq\epsilon$,  it induces a transport map $T:\cX\rightarrow\cX$ such that $\mu_{X}'=T_{\#}(\mu_X)$ and $\|\Delta(T(X),X)\|_p\leq \epsilon$ holds almost surely with respect to the randomness of $X$ and $T$. Thus, we have
$$
\Pr_{\bx\sim\mu_X}\:\Big[-\epsilon\cdot\|\bw\|_q\leq\bw^\top(T(\bx)-\bx) \leq \epsilon\cdot\|\bw\|_q\Big] \geq \Pr_{\bx\sim\mu_X}\:\Big[\|T(\bx)-\bx\|_p\leq \epsilon\Big] = 1,
$$
which implies
\begin{align*}
    \Pr_{\bx\sim\mu_{X}}\big(\bw^\top\bx - \epsilon\cdot\|\bw\|_q \geq 0\big) \leq  \Pr_{\bx'\sim\mu_{X'}}(\bw^\top\bx'\geq 0) &\leq \Pr_{\bx\sim\mu_{X}}\big(\bw^\top\bx + \epsilon\cdot\|\bw\|_q \geq 0\big).
\end{align*}
We remark that the equality can be achieved when the the transport map $T$ is constructed by perturbing the $i$-th element of any sampled $\bx\sim\mu_X$ by $\epsilon\cdot (w_i^q / \sum_i{w_i}^q)^{1/p}$, for any $i=1,2,\ldots,d$. In addition, according to the assumed Gaussian Mixture model \eqref{eq:GMM}, we have
\begin{align*}
    \Pr_{\bx\sim\mu_{X}}\big(\bw^\top\bx - \epsilon\cdot\|\bw\|_q \geq 0\big) &= \frac{1}{2}\: \Pr_{\bx\sim\cN(\btheta^*,\bSigma^*)} \big[\bw^\top\bx \geq \epsilon\cdot\|\bw\|_q\big] + \frac{1}{2}\: \Pr_{\bx\sim\cN(-\btheta^*,\bSigma^*)} \big[\bw^\top\bx \geq \epsilon\cdot\|\bw\|_q\big] \\
   &= \frac{1}{2} -  \frac{1}{2}\:\Pr_{Z\sim\cN(0,1)} \bigg[   \frac{-\epsilon\cdot\|\bw\|_q + \bw^\top\btheta^*}{ \sqrt{\bw^\top\bSigma^*\bw}} \leq Z \leq \frac{\epsilon\cdot\|\bw\|_q + \bw^\top\btheta^*}{ \sqrt{\bw^\top\bSigma^*\bw}}
   \bigg].
\end{align*}
Similarly, we have
\begin{align*}
    \Pr_{\bx\sim\mu_{X}}\big(\bw^\top\bx + \epsilon\cdot\|\bw\|_q \geq 0\big) &= \frac{1}{2} + \frac{1}{2}\:\Pr_{Z\sim\cN(0,1)} \bigg[   \frac{-\epsilon\cdot\|\bw\|_q + \bw^\top\btheta^*}{ \sqrt{\bw^\top\bSigma^*\bw}} \leq Z \leq \frac{\epsilon\cdot\|\bw\|_q + \bw^\top\btheta^*}{ \sqrt{\bw^\top\bSigma^*\bw}}
   \bigg].
\end{align*}

Therefore, we derive the explicit formulation for $\RV_\epsilon(g_{\bw})$
\begin{align}
\label{eq:feature sensitivity}
    \RV_\epsilon(g_{\bw}) = \H_2\bigg(\frac{1}{2}\bigg) -  \H_2\bigg(\frac{1}{2}-\Pr_{Z\sim\cN(0,1)} \bigg[   \frac{\bw^\top\btheta^*-\epsilon\cdot\|\bw\|_q}{ \sqrt{\bw^\top\bSigma^*\bw}} \leq Z \leq \frac{ \bw^\top\btheta^*+\epsilon\cdot\|\bw\|_q}{ \sqrt{\bw^\top\bSigma^*\bw}}
  \bigg]\bigg),
\end{align}
where $\H_2(\cdot)$ is denotes binary entropy function.

Next, given $g_{\bw}\in\cG_{\bin}$, we are going to compute the adversarial gap of $f\circ g_{\bw}$ for each $h\in\cH$. To begin with, we consider the first case $h_1(z)=z$ for any $z\in\cZ$. According to the definition of adversarial risk, we have
\begin{align*}
\AdvRisk_\epsilon(h_1\circ g_{\bw}) &=  \Pr_{(\bx,y)\sim\mu_{XY}} \big[\exists\: \bx'\in\cB(\bx, \epsilon) \:\:\text{s.t.}\:\: \sgn(\bw^\top\bx') \neq y \big] \\
&= \Pr_{(\bx,y)\sim\mu_{XY}} \bigg[\min_{\bx'\in\cB(\bx, \epsilon)} y\cdot \bw^\top\bx' \leq 0 \bigg] \\
&= \Pr_{(\bx,y)\sim\mu_{XY}} \bigg[ y\cdot \bw^\top\bx \leq - \min_{\bDelta\in\cB(\zero, \epsilon)} \bw^\top\bDelta \bigg] \\
&= \Pr_{Z\sim\cN(0,1)} \bigg[Z \leq \frac{\epsilon \|\bw\|_q - \bw^\top\btheta^*}{\sqrt{\bw^\top\bSigma^*\bw}}\bigg],
\end{align*}
where the equality is due to the fact that $\cB(\zero, \epsilon)$ is symmetric with respect to $\zero$, and the last equality holds because of the H\"older's inequality: for any $\ab, \bb\in\RR^n$, it holds that $\ab^\top\bb \geq -\|\ab\|_p\cdot \|\bb\|_q$ and the equality is achieved when $(a_i/\|\ba\|_p)^p = (b_i/\|\bb\|_q)^q$ for any $i\in\{1,2,\ldots,d\}$. 

Similarly, the standard risk can be computed as: 
\begin{align*}
\Risk(h_1\circ g_{\bw}) = \Pr_{(\bx,y)\sim\cD} \big[ \sgn(\bw^\top\bx) \neq y \big] = \Pr_{(\bx,y)\sim\cD} \big[ y\cdot \bw^\top\bx \leq 0 \big] = \Pr_{Z\sim\cN(0,1)} \bigg[Z \leq  \frac{- \bw^\top\btheta^*}{\sqrt{\bw^\top\bSigma^*\bw}}\bigg].
\end{align*}
Thus, we derive the gap between standard and adversarial risk with respect to $h_1\circ g_{\bw}$:
\begin{align*}
    \AG_\epsilon(h_1\circ g_{\bw}) = \Pr_{Z\sim\cN(0,1)} \bigg[\frac{ \bw^\top\btheta^* - \epsilon \|\bw\|_q}{\sqrt{\bw^\top\bSigma^*\bw}} \leq Z \leq \frac{\bw^\top\btheta^*}{\sqrt{\bw^\top\bSigma^*\bw}}\bigg].
\end{align*}

For the other case where $h_2(z)=-z$ for any $z\in\cZ$, note that $h_1\circ g_{\bw} = h_2\circ g_{-\bw}$ for any $g_{\bw}\in\cG_{\bin}$. Thus, a similar proof technique can be applied to compute the adversarial risk,
\begin{align*}
    \AdvRisk_\epsilon(h_2\circ g_{\bw}) &= \Pr_{Z\sim\cN(0,1)} \bigg[Z \leq \frac{\epsilon \|\bw\|_q + \bw^\top\btheta^*}{\sqrt{\bw^\top\bSigma^*\bw}}\bigg],
\end{align*}
and the adversarial gap,
\begin{align*}
    \AG_\epsilon(h_2\circ g_{\bw}) =
    \Pr_{Z\sim\cN(0,1)} \bigg[\frac{\bw^\top\btheta^*}{\sqrt{\bw^\top\bSigma^*\bw}} \leq Z \leq \frac{ \bw^\top\btheta^* + \epsilon \|\bw\|_q}{\sqrt{\bw^\top\bSigma^*\bw}}\bigg].
\end{align*}

Note that $f^*:\cX\rightarrow\cY$ is the optimal classifier based on $g_{\bw}$ that minimizes the adversarial risk of $h\circ g_{\bw}$ for any $h\in\cH$. By comparing the adversarial risk of $h_1\circ g_{\bw}$ and $h_2\circ g_{\bw}$, we have $f^*=h_1\circ g_{\bw}$, if $\bw^{\top}\btheta^*\geq 0$; $f^*=h_2\circ g_{\bw}$, otherwise. Thus, we derive the adversarial gap with respect to $f^*$ as follows
\begin{equation}
\label{eq:opt adv gap}
\AG_\epsilon(f^*) = 
\left\{
\begin{array} {ll}
\Pr_{Z\sim\cN(0,1)} \Big(\frac{ \bw^\top\btheta^* - \epsilon \|\bw\|_q}{\sqrt{\bw^\top\bSigma^*\bw}} \leq Z \leq \frac{\bw^\top\btheta^*}{\sqrt{\bw^\top\bSigma^*\bw}}\Big), & \quad\text{if} \: \bw^{\top}\btheta^*\geq 0 ;\\
\Pr_{Z\sim\cN(0,1)} \Big(\frac{\bw^\top\btheta^*}{\sqrt{\bw^\top\bSigma^*\bw}} \leq Z \leq \frac{ \bw^\top\btheta^* + \epsilon \|\bw\|_q}{\sqrt{\bw^\top\bSigma^*\bw}}\Big), & \quad\text{otherwise}.
\end{array} 
\right. 
\end{equation}

Based on \eqref{eq:opt adv gap}, we further obtain the following inequality 
$$
\AG_\epsilon(f^*) \leq \Pr_{Z\sim\cN(0,1)} \bigg[   \frac{\bw^\top\btheta^*-\epsilon\cdot\|\bw\|_q}{ \sqrt{\bw^\top\bSigma^*\bw}} \leq Z \leq \frac{ \bw^\top\btheta^*+\epsilon\cdot\|\bw\|_q}{ \sqrt{\bw^\top\bSigma^*\bw}} 
  \bigg] \leq 2\cdot \AG_\epsilon(f^*).
$$
Finally, according to the formulation of representation vulnerability \eqref{eq:feature sensitivity}, we have 
\begin{align*}
    \int_{\frac{1}{2}-\AG_\epsilon(f^*\circ g_{\bw})}^{\frac{1}{2}} \H_2'(\theta)d\theta \leq \RV_\epsilon(g_{\bw}) \leq \int_{\frac{1}{2}-\frac{1}{2}\AG_\epsilon(f^*\circ g_{\bw})}^{\frac{1}{2}} \H_2'(\theta)d\theta,
\end{align*}
which completes the proof.

\end{proof}

\subsection{Proof of Lemma \ref{lem:equal}}
\label{proof:lemma equal}

Lemma~\ref{lem:equal}, restated below, connects adversarial risk and input distribution perturbations bounded in an $\infty$-Wasserstein ball.

{\bf Lemma 3.3.\ }
Let $(\cX,\Delta)$ be the input metric space and $\cY$ be the set of labels. Assume all the examples are generated from a joint probability distribution $(X,Y)\sim\mu_{XY}$. Let $\mu_X$ be the marginal distribution of $X$. Then, for any classifier $f:\cX\rightarrow\cY$ and $\epsilon\geq 0$, we have
\begin{align*}
    \AdvRisk_\epsilon(f) = \sup_{\mu_{X'}\in\cB_{\W_\infty}(\mu_X,\epsilon)}\Pr\big[f(X')\neq Y\big],
\end{align*}
where $X'$ denotes the random variable that follows $\mu_{X'}$.
%\end{lemma} 

\begin{proof}
Our proof proves the equality by proving $\le$ inequalities in both directions. First, we prove
\begin{align}
\label{eq:dir 1}
\AdvRisk_\epsilon(f)\leq \sup_{\mu_{X'}\in\cB_{\W_\infty}(\mu_X,\epsilon)}\Pr\big[f(X')\neq Y\big].
\end{align}
For any classifier $f:\cX\rightarrow\cY$, according to Definition \ref{def:adv risk}, we have
\begin{align*}
   \AdvRisk_\epsilon(f) = \Pr_{(\bx,y)\sim\mu_{XY}} \big[\exists\:\bx'\in\cB(\bx,\epsilon) \text{ s.t. } f(\bx')\neq y\big].
\end{align*}
Since $f$ is a given deterministic function, the optimal perturbation scheme that achieves $\AdvRisk_\epsilon(f)$ essentially defines a transport map $T:\cX\rightarrow\cX$. More specifically, let $\cC_y(f) = \{\bx\in\cX: f(\bx)\neq y\}$.  Then, for any sampled pair $(\bx,y)\sim\mu_{XY}$, we can construct $T$ such that
\begin{equation*} 
T(\bx) = 
\left\{
\begin{array} {ll}
\argmin_{\bx'\in\cC_y(f)} \Delta(\bx',\bx), & \quad\text{if} \: \cC_y(f) \cap \cB(\bx,\epsilon) \neq \varnothing ;\\
\bx, & \quad\text{otherwise}.
\end{array} 
\right. 
\end{equation*}
Let $(X,Y)$ be the random variable that follows $\mu_{XY}$. By construction, it can be easily verified that $T_\#(\mu_X)\in\cB_{\W_\infty}(\mu_X,\epsilon)$ and $\AdvRisk_\epsilon(f) = \Pr \big[f(T(X))\neq Y\big]$. Therefore, we have proven \eqref{eq:dir 1}.

It remains to prove the other direction of the inequality:
\begin{align}
\label{eq:dir 2}
   \AdvRisk_\epsilon(f) \geq \sup_{\mu_{X'}\in\cB_{\W_\infty}(\mu_X,\epsilon)}\Pr\big[f(X')\neq Y\big].
\end{align}
According to the alternative definition of $\infty$-Wasserstein distance, the optimal solution $\mu_{X'}^*$ that achieves the supremum of the right hand side of \eqref{eq:dir 2} can be captured by a transport map $T^*:\cX\rightarrow\cX$ such that $\mu^*_{X'} = T^*_{\#}(\mu_X)$ and $\Delta(T^*(X),X)\leq \epsilon$ holds almost surely with respect to the randomness of $X$ and $T^*$. Thus, we have
\begin{align*}
    \Pr\big[f(T^*(X))\neq Y\big] &= \Pr_{(\bx,y)\sim\mu_{XY}}\big[f(T^*(\bx))\neq y\big] \\
    &= \Pr_{(\bx,y)\sim\mu_{XY}}\big[ \Delta(T^*(\bx),\bx)\leq\epsilon \:\text{ and }\: f(T^*(\bx))\neq y\big] \\
    &\leq 1 - \Pr_{(\bx,y)\sim\mu_{XY}} \big[\forall\:\bx'\in\cB(\bx,\epsilon) \text{ s.t. } f(\bx')= y\big] = \AdvRisk_\epsilon(f).
\end{align*}
Therefore, we have proven the second direction and completed the proof.
\end{proof}

\subsection{Proof of Theorem \ref{thm:general lower bound}}
\label{proof:general thm}

Theorem \ref{thm:general lower bound}, restated below, gives a lower bound for the adversarial risk for any downstream classifier in terms of the worst-case mutual information between the representation's input and output distributions. 

{\bf Theorem \ref{thm:general lower bound}.\ }
Let $(\cX,\Delta)$ be the input metric space, $\cY$ be the set of labels and $\mu_{XY}$ be the underlying joint probability distribution. Assume the marginal distribution of labels $\mu_Y$ is a uniform distribution over $\cY$. Consider the feature space $\cZ$ and the set of downstream classifiers $\cH=\{h:\cZ\rightarrow\cY\}$. Given $\epsilon\geq 0$, for any $g:\cX\rightarrow\cZ$, we have
\begin{align*}
\inf_{h\in\cH} \AdvRisk_\epsilon(h\circ g) \geq 1 - \frac{\I(X;Z)-\RV_\epsilon(g) + \log2}{\log|\mathcal{Y}|},
\end{align*}
where $X$ is the random variable that follows the marginal distribution of inputs $\mu_X$ and $Z=g(X)$.

\vspace{2ex}
Before starting the proof, we state two useful lemmas on Markov chains. A Markov chain is defined to be a collection of random variables $\{X_t\}_{t\in\ZZ}$ with the property that given the present, the future is conditionally independent of the past. Namely,
$$
     \Pr(X_t=j|X_0=i_0,X_1=i_1,...,X_{(t-1)}=i_{(t-1)})=\Pr(X_t=j|X_{(t-1)}=i_{(t-1)}). 
$$

\begin{lemma}[Fano's Inequality]
\label{lem:Fano}
Let $X$ be a random variable uniformly distributed over a finite set of outcomes $\cX$. For any estimator $\hat{X}$ such that $X\rightarrow Y\rightarrow \hat{X}$ forms a Markov chain, we have
\begin{align*}
    \Pr(\hat{X}\neq X) \geq 1 - \frac{\I(X;\hat{X})-\log 2}{\log |\cX|}.
\end{align*}
\end{lemma}

\begin{lemma}[Data-Processing Inequality]
\label{lem:data process}
For any Markov chain $X\rightarrow Y\rightarrow Z$, we have 
$$
\I(X;Y) \geq I(X;Z) \quad \text{and} \quad \I(Y;Z) \geq \I(X;Z).
$$
\end{lemma}

Chapter 2 in \citet{cover2012elements} provides proofs of Lemmas \ref{lem:Fano} and \ref{lem:data process}. 

\begin{proof}[Proof of Theorem \ref{thm:general lower bound}]
For any classifier $h:\cZ\rightarrow\cY$, according to Lemma \ref{lem:equal}, we have
\begin{align}
\label{eq:lemma 3.4}
    \AdvRisk_\epsilon (h\circ g) = \sup_{\mu_{X'}\in\cB_{\W_\infty}(\mu_X,\epsilon)} \Pr \big[ h (g(X')) \neq Y \big].
\end{align}
Let $\mu_{X'}\in\cB_{\W_\infty}(\mu_X,\epsilon)$ be a probability measure over $(\cX,\Delta)$. According to the alternative definition of $\infty$-Wasserstein distance using optimal transport, $\mu_{X'}$ corresponds to a transport map $T:\cX\rightarrow\cX$ such that $\mu_{X'} = T_{\#}(\mu_X)$. Thus, for any given $\mu_{X'}\in\cB_{\W_\infty}(\mu_X,\epsilon)$ and $h\in\cH$, we have the Markov chain
\begin{align*}
    Y\rightarrow & X \xrightarrow{T} X' \xrightarrow{g} g(X') \xrightarrow{h} (h\circ g)(X').
\end{align*}
where $X, Y$ are random variables for input and label distributions respectively. The first Markov chain $Y\rightarrow X$ can be understood as a generative model for generating inputs according to the conditional probability distribution $\mu_{X|Y}$. Therefore, applying Lemmas \ref{lem:Fano} and \ref{lem:data process}, we obtain the inequality,
\begin{align}
\label{eq:fano adv risk}
    \Pr \big[ h(g(X')) \neq Y \big] \geq 1 - \frac{\I\big(Y; (h\circ g)(X')\big) + \log2}{\log|\mathcal{Y}|} \geq 1 - \frac{\I\big(X'; g(X')\big) + \log2}{\log|\mathcal{Y}|}.
\end{align}
Taking the supremum over the distribution of $X'$ in $\cB_{\W_\infty}(\mu_X,\epsilon)$ and infimum over $h\in\cH$ on both sides of \eqref{eq:fano adv risk} yields
\begin{align*}
    \inf_{h\in\cH} \big[ \AdvRisk_\epsilon(h\circ g) \big] &= \inf_{h\in\cH}\: \sup_{\mu_{X'}\in\cB_{\W_\infty}(\mu_X,\epsilon)}\: \Pr \big[ h (g(X')) \neq Y \big] \\ 
    &\geq 1 - \frac{\inf_{\mu_{X'}\in\cB_{\W_\infty}(\mu_X,\epsilon)}\I\big(X';  g(X')\big) + \log 2}{\log|\mathcal{Y}|} \\
    &= 1 - \frac{\I(X;g(X))-\RV_\epsilon(g) + \log2}{\log|\mathcal{Y}|},
\end{align*}
where the first equality is due to \eqref{eq:lemma 3.4} and the inequality holds because of \eqref{eq:fano adv risk}. Thus, we completed the proof.
\end{proof}

\section{Algorithm for Estimating the Worst-Case Mutual Information}
\label{sec:alg code}

This section presents the pseudocode of our heuristic algorithm for solving the empirical estimation problem \eqref{eq:worst case MI}. More specifically, given a training sample set $\cS_{\text{train}}$, our algorithm alternatively optimizes for the worst-case input perturbations using projected gradient descent (Algorithm \ref{alg:worst-case perturbation}) and conducts gradient ascent for the network parameters $\theta$ (training phase in Algorithm \ref{alg:estimator training}). Based on the best parameter $\theta_{\text{opt}}$ selected from the training phase, our algorithm then estimates the worst-case mutual information with respect to the given representation $g$ using a testing sample set $\cS_{\text{test}}$ (testing phase in Algorithm \ref{alg:estimator training}).
Since we only have assess to a finite set of data sampled from $\mu_X$, we use an additional testing phase in Algorithm \ref{alg:estimator training} to minimize the overfitting effect of
the training samples on the optimal network parameter $\theta_{\text{opt}}$ for mutual information neural estimation (MINE).

Moreover, we adopt the negative sampling scheme \citep{hjelm2018learning} to estimate the expectation term with respect to $\hat\mu_X^{(m)}\otimes \hat\mu_Z^{(m)}$ in mutual information neural estimation for better performance. Here, the pairing scheme defines a correspondence from each input to a set of inputs for a given sample set. To be more specific, given a set of samples $\{\bx_i\}_{i\in[B]}$, a pairing scheme with negative sampling size $N\leq B$ corresponds to a set of vectors $\{\bpi_i\}_{i\in[B]}$ such that each $\bpi_i$ is a randomly selected subset from $\{1,2,\ldots,B\}$ with size $N$, and 
$\pi_{ij}$ denotes the $j$-th element of $\bpi_i$.
Compared with the algorithm in \citet{hjelm2018learning} for estimating standard mutual information, Algorithm 2 requires additional $B\cdot S$ steps of forward and backward propagations with respect to the input for finding the worst-case input perturbations in each iteration.

\begin{algorithm}[thb]
\caption{Heuristic Search for Worst-Case Input Perturbations}
\label{alg:worst-case perturbation}
\textbf{Input:} samples $\{\bx_i\}_{i\in[B]}$, representation $g$, MINE estimator $T_\theta$, paring scheme $\{\bpi_i\}_{i\in[B]}$, perturbation strength $\epsilon$ in $\ell_p$ \par
\textbf{Hyperparameters:} negative sampling size $N$, number of iterations $S$, step size $\eta_a$ \par
\begin{algorithmic}[1] 
\STATE Initialize $\{\bx'_i\}_{i\in[B]} \leftarrow \{\bx_i\}_{i\in[B]}$
\FOR{$s=1,2,\ldots,S$}
\STATE $J(\bx'_1,\ldots,\bx'_B,\theta) \leftarrow \frac{1}{B}\sum_{i=1}^{B}T_\theta\big(\bx'_i,g(\bx'_i)\big)-\log\big(\frac{1}{BN}\sum_{i=1}^{B}\sum_{j=1}^{N}\exp\big[{T_\theta\big(\bx'_i,g(\bx'_{\pi_{ij}})\big)}\big]\big)$
\FOR{$i=1,2,\ldots,B$}
\STATE $\bx'_i \leftarrow \cP_{\cB(\bx_i, \epsilon)} \big[ \bx'_i - \eta_a \cdot \nabla _{\bx'_i} J(\bx'_1,\ldots,\bx'_B,\theta) \big]$ \algorithmiccomment{\qquad //\quad $\cP_{\cB(\bx_i, \epsilon)}$ denotes the projection operator onto $\cB(\bx_i, \epsilon)$}
\ENDFOR
\ENDFOR
\STATE $V_1\leftarrow J(\bx'_1,\ldots,\bx'_B,\theta)$
\STATE $V_2\leftarrow \nabla_\theta J(\bx'_1,\ldots,\bx'_B,\theta)$
\end{algorithmic}
\textbf{Output:}  $\{V_1, V_2\}$
\end{algorithm}

\begin{algorithm}[thb]
\caption{Empirical Estimation of Worst-Case Mutual Information}
\label{alg:estimator training}
\textbf{Input:} training and testing sample sets $(\cS_{\text{train}}, \cS_{\text{test}})$ sampled from $\mu_X$, representation $g$, perturbation strength $\epsilon$ in $\ell_p$ \par
\textbf{Hyperparameters:} number of training epochs $T$, step size $\eta_e$, number of testing mini-batches $K$ \par
\begin{algorithmic}[1] 
\STATE\algorithmiccomment{\textbf{//\: Training Phase}} \\
\STATE  $\theta_1\leftarrow$ initialize network parameter for MINE estimator
\FOR{$t=1,2,\ldots,T$}    
\STATE $\{\bx_i\}_{i\in[B]},\{\bpi_i\}_{i\in[B]}\leftarrow$ randomly generate a batch of $B$ training samples and a pairing scheme
\STATE $\{V_1(t), V_2(t)\} \leftarrow$ Algorithm~\ref{alg:worst-case perturbation}$\big(\{\bx_i\}_{i\in[B]},g,T_{\theta_t},\{\bpi_i\}_{i\in[B]},\epsilon\big)$
\STATE $\theta_{t+1} \leftarrow \theta_t + \eta_e\cdot V_2(t)$
\ENDFOR
\STATE $\theta_{\text{opt}}\leftarrow\argmax\big\{t\in[T]:V_1(t)\big\}$  \algorithmiccomment{\qquad //\quad choose the best parameter $\theta_{\text{opt}}$ based on history}
\STATE\algorithmiccomment{\textbf{//\: Testing Phase}} \\
\STATE Randomly split the testing set $\cS_{\text{test}}$ into $K$ mini-batches
$\{\cS_1,\ldots,\cS_K\}$ with equal size
\FOR{$k=1,2,\ldots,K$}
\STATE $\{\bpi(\bx)\}_{\bx\in\cS_k}\leftarrow$ randomly generate a pairing scheme with respect to $\cS_k$
\STATE $\{V_1(k), V_2(k)\} \leftarrow$ Algorithm~\ref{alg:worst-case perturbation}$\big(\cS_k,g,T_{\theta_{\text{opt}}}, \{\bpi(\bx)\}_{\bx\in\cS_k}, \epsilon\big)$
\ENDFOR
\STATE $\hat{I}_{\text{worst}} \leftarrow\frac{1}{K}\sum_{k=1}^K V_1(k)$ 
\end{algorithmic}
\textbf{Output:} $\hat{I}_{\text{worst}}$
\end{algorithm}

\section{Worst-case Mutual Information for Individual Neuron Features}
\label{section:appendix:tensorization}

The following tensorization inequality \citep{scarlett2019introductory}
characterizes the connection between the mutual information of individual neuron features and that of the whole representation. According to Theorem \ref{thm:general lower bound}, such connection suggests the necessity of enough worst-case mutual information for each individual neuron.
\newpage

\begin{lemma}[Tensorization of Mutual Information]
% \label{lemma:tensorization}
Let $\bZ=(Z_1, ..., Z_n)$ be a product distributions over random variables. If $Z_1$, ..., $Z_n$ are mutually independent conditioned on $X$, then
\begin{align*}
I(X;\bZ) \leq \sum_{i=1}^{n} I(X;Z_i) 
\end{align*}
\end{lemma}

Suppose neurons within a single layer have no interconnection, then each neuron's output is mutually independent conditioned on the model input. 
If a perturbation imposed on the input distribution makes the perturbed mutual information $I(X';Z'_i)$ relatively low for each neuron, then the perturbed mutual information with respect to the entire layer $I(X';\bZ')$ will also be low, which further implies a low adversarial accuracy for any downstream classifier based on Theorem \ref{thm:general lower bound}.

\newpage
\section{Experiments}
\label{section:appendix:experiment}

\subsection{Implementation Details}
\label{section:appendix implement}

Here, we provide additional implementation details of our experiments presented in Section~\ref{sec:experiments}.

\shortsection{Model architectures} For all experiments, we follow \citet{hjelm2018learning} in implementing the MINE estimator. We adopt the \textit{encode-and-dot-product} model architecture in \citet{hjelm2018learning} which maps $\bx$ and $\bz$ respectively to two high-dimensional vectors and then takes the dot-product to calculate the output. 
The basic modules used in our experiments are listed in Table \ref{tab:modules}. A slight difference in training the feature (encoder) is that \citet{hjelm2018learning} shares parameters between parts of the mutual information estimator and the encoder, while we separate the two parts completely to be consistent with our mutual information estimation experiments.

\begin{table}[!hp]
\centering
\begin{tabular}{@{}ll@{}}
\toprule
\multicolumn{1}{c|}{Module} & \multicolumn{1}{c}{Structure} \\ \midrule
\multicolumn{1}{l|}{Feature Extractor} & 
$\text{Conv}(64,4\times4, 2) \rightarrow
\text{Conv}(128,4\times4, 2) \rightarrow
\text{Conv}(256,4\times4, 2) \rightarrow
\text{FC}(1024) \rightarrow
\text{FC}(64)$ \\
\multicolumn{1}{l|}{Top Classifier (MLP)} & 
$\text{FC}(200) \rightarrow
\text{FC}(10)$ \\
\multicolumn{1}{l|}{Top Classifier (Linear)} & 
$\text{FC}(10)$ \\
\multicolumn{1}{l|}{Baseline-H} & Feature Extractor $\rightarrow$ Top Classifier (MLP) \\  \midrule
\multicolumn{1}{l|}{Estimator Part 1} & $\text{Conv}(64,4\times4, 2) \rightarrow
\text{Conv}(128,4\times4, 2) \rightarrow
\text{Conv}(256,4\times4, 2)$ \\
\multicolumn{1}{l|}{Estimator Part 2} & $\text{Conv}(2048,1\times1, 1) \rightarrow
\text{Conv}(2048,1\times1, 1)$ \\
\multicolumn{1}{l|}{Estimator} & 
($\bx$ $\rightarrow$ Estimator Part 1 $\rightarrow$ Estimator Part 2) $\cdot$
($\bz$ $\rightarrow$ Estimator Part 2)
\\
\bottomrule
\end{tabular}
\caption{Basic model structures used in our experiments. Batch-normalization and ReLU activation are used between layers (not including the output of each module). Shortcut-connection is omitted for Estimator Part 2. For scalar feature $z$, Estimator Part 2 is replaced by an identity mapping. Average operation is needed in the dot-product operation of Estimator. For more details, see \citet{hjelm2018learning}}
\label{tab:modules}
\end{table}

\shortsection{Hyperparameters} We use simple hyperparameter settings to control their effect on our various ablation experiments. We use $l_\infty$ constrained perturbations and PGD attack \citep{madry2017towards} on all datasets.
For CIFAR-10, we set the radius $\epsilon=8/255$ and use $7$ attack steps with step size $0.01$. For MNIST, we set the radius $\epsilon=0.3$ and use $10$ attack steps with step size $0.1$. For Fashion-MNIST, we set the radius $\epsilon=0.1$ and use $10$ attack steps with step size $0.02$. For SVHN, we set the radius $\epsilon=4/255$ and use $10$ attack steps with step size $0.005$.
The batch size is set as $128$ for both datasets, and our results are consistent with different batch sizes between $128$ to $512$ (we did not test other sizes). A total training epochs of $200$ is set for VGG, ResNet, and DenseNet, with an initial learning rate of $0.1$ which decays by a factor of $10$ every $50$ epochs. For the Baseline-H model and the similar mutual information estimator, we set the training epoch to $300$ and use a fixed learning rate of $0.0001$ as in \citet{hjelm2018learning}.

\subsection{Additional Results} \label{section:appendix:additional results}

\shortsection{Results for MNIST, Fashion-MNIST, and SVHN}
We present the downstream classification results for MNIST, Fashion-MNIST, and SVHN in Table~\ref{tab:main:mnist}, \ref{tab:main:fashion-mnist}, \ref{tab:main:svhn}. These results support similar conclusions as those drawn from CIFAR-10 dataset in Table~\ref{tab:main:cifar}. That is, our training principle always produces representations that have significantly better adversarial accuracy for downstream adversarial classification. In many cases, our training principle also produces representations that have better natural accuracy, despite the worst-case situation that our training principle considers.

\begin{table*}[!htb]
\centering
\begin{tabular}{@{}llSSSS@{}}
\toprule
 &  & \multicolumn{2}{c}{MLP $h$} & \multicolumn{2}{c}{Linear $h$} \\ %\midrule
\multicolumn{1}{c}{\text{Representation ($g$)}} & \multicolumn{1}{c}{\text{Classifier ($h$)}} & \multicolumn{1}{c}{\text{Natural}} & \multicolumn{1}{c}{\text{Adversarial}} & \multicolumn{1}{c}{\text{Natural}} & \multicolumn{1}{c}{\text{Adversarial}} \\ \midrule
\multicolumn{1}{c}{\citet{hjelm2018learning}} & \multicolumn{1}{c}{Standard} & \bfseries 96.96 \pm 0.35 & 0.00 \pm 0.00 & \bfseries 84.88 \pm 1.11 & 0.00 \pm 0.00 \\
\multicolumn{1}{c}{\citet{hjelm2018learning}} & \multicolumn{1}{c}{Robust} & 44.99 \pm 14.49 & 16.70 \pm 2.22 & 11.35 \pm 0.00 & 11.35 \pm 0.00 \\
\multicolumn{1}{c}{Ours} & \multicolumn{1}{c}{Standard} & 96.67 \pm 0.12 & 9.97 \pm 1.88 & 82.10 \pm 0.37 & 3.69 \pm 0.55 \\
\multicolumn{1}{c}{Ours} & \multicolumn{1}{c}{Standard (E.S.)} & 94.68 \pm 0.93 & 12.79 \pm 2.24 & 77.72 \pm 2.47 & 4.73 \pm 1.29 \\
\multicolumn{1}{c}{Ours} & \multicolumn{1}{c}{Robust} & 95.05 \pm 0.19 & \bfseries 60.64 \pm 1.82 & 73.99 \pm 1.16 & \bfseries 30.55 \pm 1.34 \\ \midrule
\multicolumn{2}{c}{Fully-Supervised Standard} & 99.13 \pm 0.23 & 0.45 \pm 0.33 & 99.13 \pm 0.04 & 0.00 \pm 0.00 \\
\multicolumn{2}{c}{Fully-Supervised Robust} & 99.25 \pm 0.05 & 95.73 \pm 0.09 & 99.21 \pm 0.06 & 95.29 \pm 0.18 \\ \bottomrule
\end{tabular}
\caption{Comparisons of different representation learning methods for downstream classification on MNIST. \textit{E.S.} denotes early stopping under the criterion of the best adversarial accuracy. We present the mean accuracy and the standard deviation over $4$ repeated trials. }
\label{tab:main:mnist}
\end{table*}

\begin{table*}[!htb]
\centering
\begin{tabular}{@{}llSSSS@{}}
\toprule
 &  & \multicolumn{2}{c}{MLP $h$} & \multicolumn{2}{c}{Linear $h$} \\ %\midrule
\multicolumn{1}{c}{\text{Representation ($g$)}} & \multicolumn{1}{c}{\text{Classifier ($h$)}} & \multicolumn{1}{c}{\text{Natural}} & \multicolumn{1}{c}{\text{Adversarial}} & \multicolumn{1}{c}{\text{Natural}} & \multicolumn{1}{c}{\text{Adversarial}} \\ \midrule
\multicolumn{1}{c}{\citet{hjelm2018learning}} & \multicolumn{1}{c}{Standard} & 89.58 \pm 0.13 & 0.00 \pm 0.00 & 85.93 \pm 0.26 & 0.00 \pm 0.00 \\
\multicolumn{1}{c}{\citet{hjelm2018learning}} & \multicolumn{1}{c}{Robust} & 48.61 \pm 4.96 & 14.95 \pm 0.79 & 10.00 \pm 0.00 & 10.00 \pm 0.00 \\
\multicolumn{1}{c}{Ours} & \multicolumn{1}{c}{Standard} & \bfseries 90.45 \pm 0.19 & 5.38 \pm 1.00 & \bfseries 87.37 \pm 0.10 & 18.20 \pm 2.87 \\
\multicolumn{1}{c}{Ours} & \multicolumn{1}{c}{Standard (E.S.)} & 81.66 \pm 0.18 & 29.71 \pm 2.00 & 86.27 \pm 0.64 & 23.40 \pm 2.65 \\
\multicolumn{1}{c}{Ours} & \multicolumn{1}{c}{Robust} & 84.31 \pm 0.29 & \bfseries 70.44 \pm 3.62 & 81.05 \pm 0.30 & \bfseries 61.33 \pm 0.49 \\ \midrule
\multicolumn{2}{c}{Fully-Supervised Standard} & 92.09 \pm 0.23 & 0.00 \pm 0.00 & 85.93 \pm 0.26 & 0.00 \pm 0.00 \\
\multicolumn{2}{c}{Fully-Supervised Robust} & 87.94 \pm 0.18 & 77.59 \pm 0.38 & 88.05 \pm 0.46 & 77.15 \pm 0.24 \\ \bottomrule
\end{tabular}
\caption{Comparisons of different representation learning methods for downstream classification on Fashion-MNIST. \textit{E.S.} denotes early stopping under the criterion of the best adversarial accuracy. We present the mean accuracy and the standard deviation over $4$ repeated trials. }
\label{tab:main:fashion-mnist}
\end{table*}

\begin{table*}[!htb]
\centering
\begin{tabular}{@{}llSSSS@{}}
\toprule
 &  & \multicolumn{2}{c}{MLP $h$} & \multicolumn{2}{c}{Linear $h$} \\ %\midrule
\multicolumn{1}{c}{\text{Representation ($g$)}} & \multicolumn{1}{c}{\text{Classifier ($h$)}} & \multicolumn{1}{c}{\text{Natural}} & \multicolumn{1}{c}{\text{Adversarial}} & \multicolumn{1}{c}{\text{Natural}} & \multicolumn{1}{c}{\text{Adversarial}} \\ \midrule
\multicolumn{1}{c}{\citet{hjelm2018learning}} & \multicolumn{1}{c}{Standard} & 50.15 \pm 0.89 & 0.00 \pm 0.00 & 38.94 \pm 1.52 & 0.00 \pm 0.00 \\
\multicolumn{1}{c}{\citet{hjelm2018learning}} & \multicolumn{1}{c}{Robust} & 19.59 \pm 0.00 & 19.59 \pm 0.00 & 19.59 \pm 0.00 & 19.59 \pm 0.00 \\
\multicolumn{1}{c}{Ours} & \multicolumn{1}{c}{Standard} & \bfseries 74.32 \pm 0.49 & 26.29 \pm 1.41 & \bfseries 58.37 \pm 0.54 & 21.62 \pm 0.91 \\
\multicolumn{1}{c}{Ours} & \multicolumn{1}{c}{Standard (E.S.)} & 71.85 \pm 0.59 & 29.59 \pm 0.83 & 54.76 \pm 0.86 & 25.00 \pm 0.53 \\
\multicolumn{1}{c}{Ours} & \multicolumn{1}{c}{Robust} & 68.25 \pm 0.83 & \bfseries 40.23 \pm 0.83 & 49.04 \pm 0.79 & \bfseries 30.56 \pm 0.38 \\ \midrule
\multicolumn{2}{c}{Fully-Supervised Standard} & 91.97 \pm 0.13 & 9.77 \pm 1.58 & 91.33 \pm 0.15 & 9.29 \pm 1.73 \\
\multicolumn{2}{c}{Fully-Supervised Robust} & 90.14 \pm 0.83 & 65.35 \pm 0.44 & 89.60 \pm 0.54 & 64.48 \pm 1.06 \\ \bottomrule
\end{tabular}
\caption{Comparisons of different representation learning methods for downstream classification on SVHN. \textit{E.S.} denotes early stopping under the criterion of the best adversarial accuracy. We present the mean accuracy and the standard deviation over $4$ repeated trials. }
\label{tab:main:svhn}
\end{table*}

\shortsection{Saliency maps of internal features}
In section~\ref{section:exp1}, we evaluated the internal feature vulnerability of all the convolutional kernels in the second layer of Baseline-H. Here, we further visualize the saliency maps of the those internal features to evaluate the underlying correlations. As shown in Figure~\ref{fig:saliencyKernel}, features in robust model have less noisy saliency maps, which is consistent with the observations of lower representation vulnerability shown in Figure~\ref{fig:neuron_historgram}.

\shortsection{Saliency maps of learned representations}
More comparison results of saliency maps are given in Figure~\ref{fig:saliency:additional}. The saliency maps of representations learned using our unsupervised training method shows comparable interpretability results to the models learned using fully-supervised adversarial training. Saliency maps computed by different losses also show consistent interpretability results. This indicates that our training principle indeed produces adversarially robust representations.

\newpage
\null
\newpage

\begin{multicols}{2}
\begin{figure}[H]
    \centering
    \includegraphics[width=0.9\columnwidth]{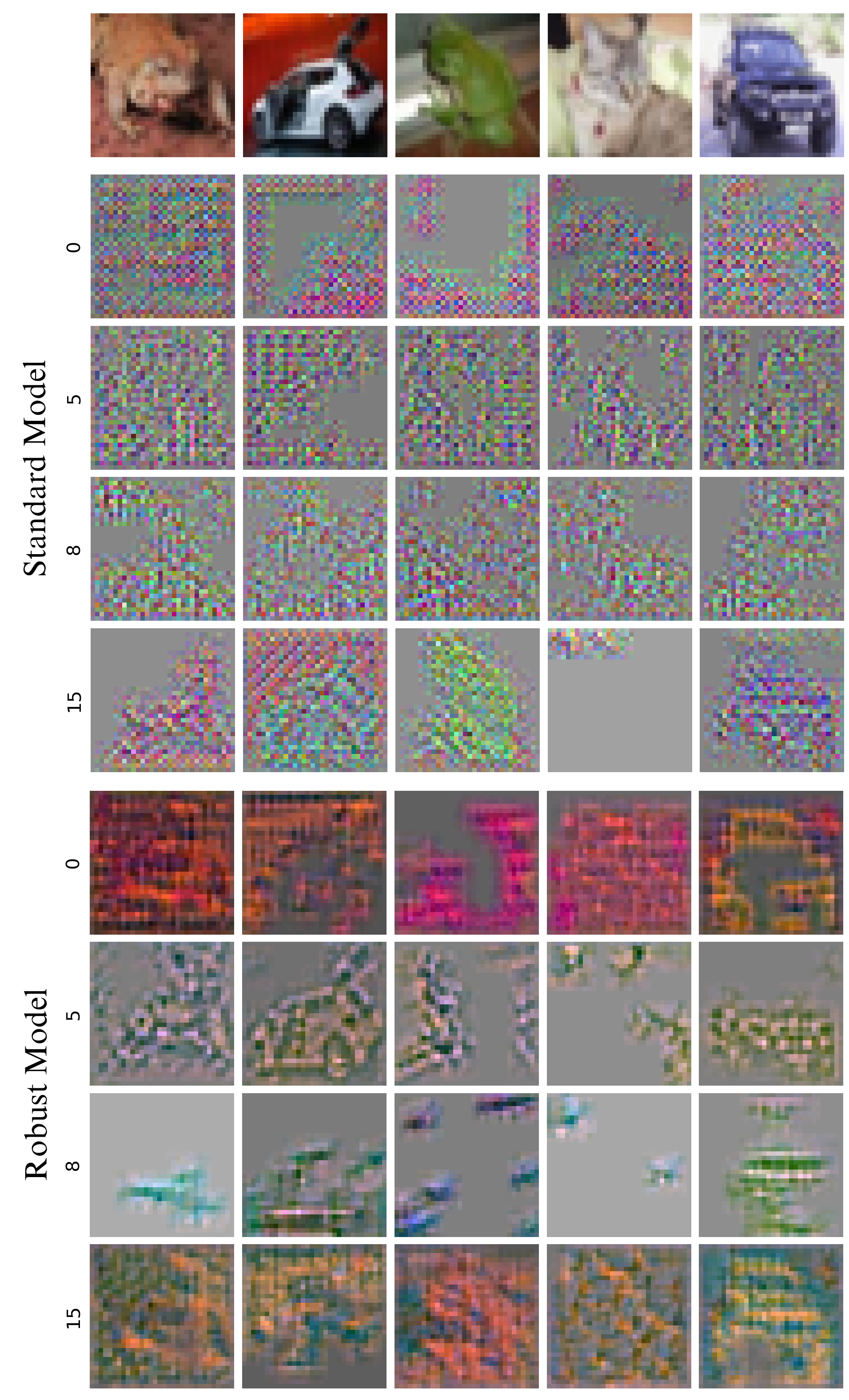}
    \caption{Saliency maps of four arbitrarily selected features in the second convolutional layer of Baseline-H. The feature saliency map is computed by the gradient of a kernel's averaged activation over a input image. Each row presents saliency maps of a specific convolutional kernel.}
    \label{fig:saliencyKernel}
\end{figure}

\begin{figure}[H]
    \centering
    \includegraphics[width=0.9\columnwidth]{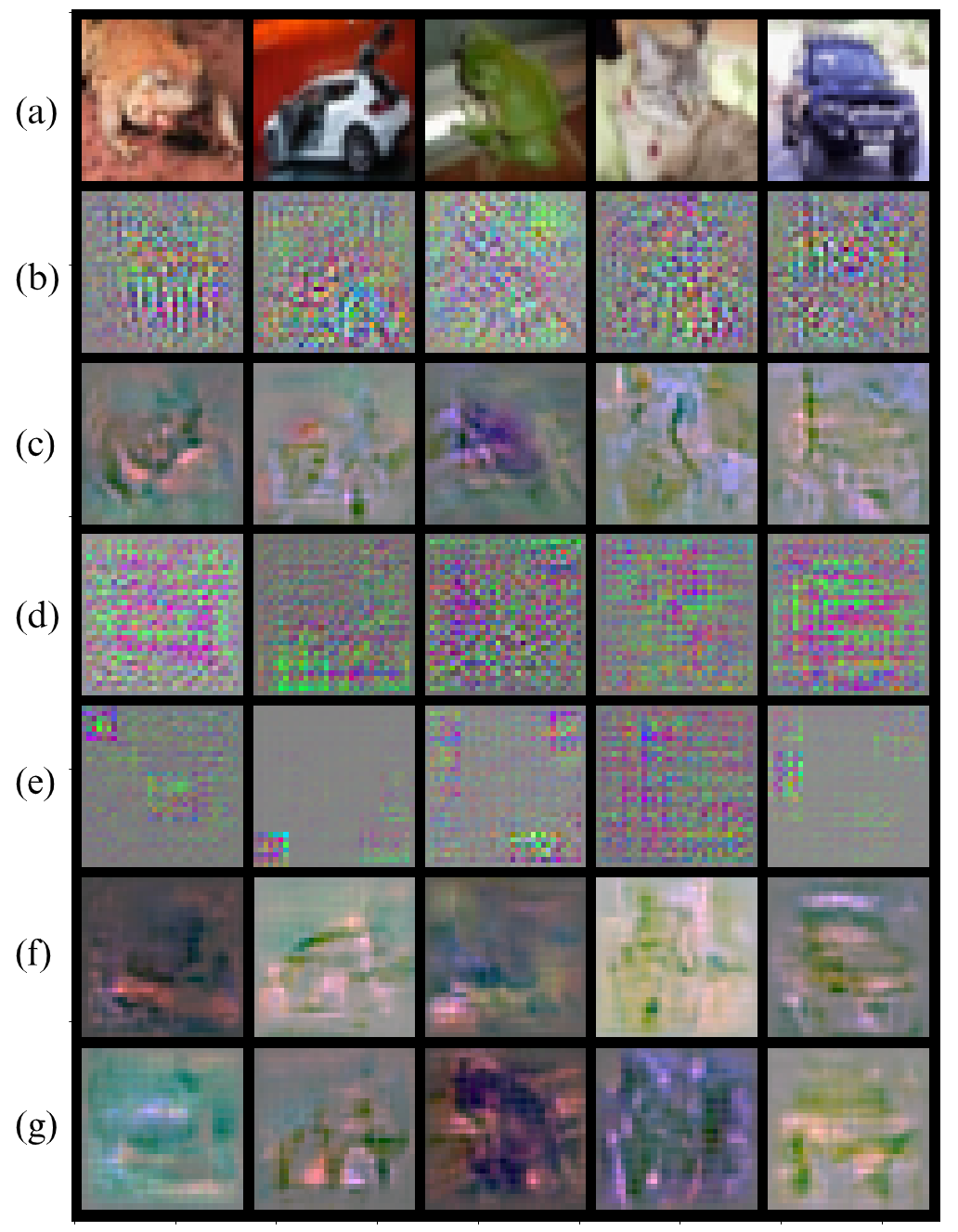}
    \caption{Saliency maps of different models on CIFAR-10: (a) original images (b) fully-supervised standard model (c) fully-supervised robust model (d) representations learned using \citet{hjelm2018learning} with cross-entropy loss (e) representations learned using \citet{hjelm2018learning} with mutual information maximization loss (f) representations learned using our method with cross-entropy loss (g) representations learned using our method with mutual information maximization loss.}
    \label{fig:saliency:additional}
\end{figure}
\end{multicols}